\documentclass[nohyperref]{article}

\usepackage{microtype}
\usepackage{graphicx}
\usepackage{booktabs} 

\usepackage{hyperref}
\hypersetup{
    colorlinks,
    linkcolor={red!50!black},
    citecolor={blue!50!black},
    urlcolor={blue!80!black}
}

\usepackage[accepted]{icml2024}

\usepackage{amsmath}
\usepackage{amssymb}
\usepackage{mathtools}
\usepackage{amsthm}

\usepackage[capitalize,noabbrev]{cleveref}

\usepackage{bbm, dsfont}
\usepackage{subcaption}
\usepackage{tikz}
\usepackage{booktabs,rotating,multirow}
\usepackage{adjustbox}
\usepackage{enumitem}
\usepackage[hang,flushmargin]{footmisc}
\usepackage[T1]{fontenc}
\usepackage[scaled=0.85]{sourcecodepro}
\usepackage[leftcaption]{sidecap}
\sidecaptionvpos{figure}{t}

\theoremstyle{plain}
\newtheorem{proposition}{Proposition} 
\newtheorem{lemma}{Lemma}
\newtheorem{corollary}{Corollary} 
\theoremstyle{definition}
\newtheorem{definition}{Definition}

\theoremstyle{remark}

\newcommand{\squeeze}{\looseness=-1}

\newcommand{\red}[1]{} 
\newcommand{\blue}[1]{} 
\newcommand{\green}[1]{} 
\newcommand{\purple}[1]{} 
\newcommand{\orange}[1]{} 
\newcommand{\cyan}[1]{} 

\newcommand{\crimson}[1]{} 
\newcommand{\yellow}[1]{} 

\newcommand\todo[1]{{\red{TODO: {#1}}}}

\newcommand\extended[1]{{\orange{EXTENDED: {#1}}}}

\newcommand{\naive}{na\"{\i}ve}
\newcommand{\Naive}{Na\"{\i}ve}

\newcommand\expect[2]{\mathbbm{E}_{#1}{\left[ {#2} \right]}}

\newcommand{\one}[1]{\mathds{1}{\{{#1}\}}}

\DeclareMathOperator*{\argmax}{argmax}
\DeclareMathOperator*{\argmin}{argmin}
\DeclareMathOperator*{\supp}{supp}
\renewcommand*{\d}{\mathop{}\!\mathrm{d}}

\newcommand{\inner}[1]{{\langle #1 \rangle}}

\newcommand{\X}{{\cal{X}}}

\newcommand{\R}{\mathbb{R}}

\newcommand{\yhat}{{\hat{y}}}
\newcommand{\phat}{{\hat{p}}}

\newcommand{\ytilde}{{\tilde{y}}}
\newcommand{\ptilde}{{\tilde{p}}}
\newcommand{\atilde}{{\tilde{a}}}
\newcommand{\mtilde}{{\tilde{m}}}
\newcommand{\xbar}{{\bar{x}}}

\newcommand{\grp}{{z}}
\newcommand{\costu}{{c}}
\newcommand{\smplst}{{S}}
\newcommand{\loss}{{\ell}}

\newcommand{\acc}{{\mathtt{acc}}}
\newcommand{\err}{{\mathtt{err}}}
\newcommand{\prc}{{\mathtt{prc}}}
\newcommand{\softprc}{{\widetilde{\prc}}}
\newcommand{\rcl}{{\mathtt{rcl}}}

\newcommand{\base}{{\mu}}
\newcommand{\thresh}{{\tau}}
\newcommand{\score}{{\phi}}
\newcommand{\sbsdy}{{s}}
\newcommand{\sig}{\xi} 
\newcommand{\sigalt}{{\sigma}}
\newcommand{\temp}{{\tau}}
\newcommand{\bias}{{B}}

\newcommand{\dataset}[1]{{\texttt{#1}}}
\newcommand{\adult}{\dataset{adult}}
\newcommand{\bank}{\dataset{bank}}
\newcommand{\method}[1]{{\fontfamily{lmtt}\selectfont{{#1}}}}
\newcommand{\naivemthd}{{\method{\naive}}}
\newcommand{\semi}{{\method{semi}}}
\newcommand{\strat}{{\method{strat}}}
\newcommand{\stratx}{\method{strat}$_x$}
\newcommand{\stratnoz}{{\method{strat}$_{x \setminus \grp}$}}
\newcommand{\stratindep}{{\method{strat}$_{\yhat \perp \grp}$}}
\newcommand{\feature}[1]{\texttt{#1}}

\begin{document}

\twocolumn[
\icmltitle{Classification Under Strategic Self-Selection}



\icmlsetsymbol{equal}{*}

\begin{icmlauthorlist}
\icmlauthor{Guy Horowitz}{equal,technion}
\icmlauthor{Yonatan Sommer}{equal,technion}
\icmlauthor{Moran Koren}{bgu}
\icmlauthor{Nir Rosenfeld}{technion}
\end{icmlauthorlist}

\icmlaffiliation{technion}{Faculty of Computer Science, Technion -- Israel Institute of Technology}
\icmlaffiliation{bgu}{Department of Industrial Engineering and Management, Ben Gurion University}

\icmlcorrespondingauthor{Nir Rosenfeld}{nirr@cs.technion.ac.il}

\icmlkeywords{Machine Learning, ICML}

\vskip 0.3in
]



\printAffiliationsAndNotice{\icmlEqualContribution} 

\begin{abstract}

When users stand to gain from certain predictive outcomes,
they are prone to act strategically to obtain predictions that are favorable.
Most current works consider strategic behavior that manifests as users  modifying their features;
instead, we study a novel setting in which users decide whether to even participate (or not),
this in response to the learned classifier.
Considering learning approaches of increasing strategic awareness,
we investigate the effects of user self-selection on learning,
and the implications of learning on the composition of the self-selected population.
Building on this, we propose a differentiable framework for learning under self-selective behavior, which can be optimized effectively.
We conclude with experiments on real data and simulated behavior that complement our analysis and demonstrate the utility of our approach.
\squeeze
\end{abstract}

\section{Introduction} \label{sec:intro}
Machine learning is increasingly being used for informing decisions regarding humans;
some common examples include
loan approvals, university admissions, job hiring, welfare  benefits, and healthcare programs.
In these domains, learned models often serve as `gatekeepers',
used for screening potential candidates in order to determine their qualification (e.g., for a job, loan, or program).
This approach is based on the premise that more accurate models should provide better screening---which in turn should enable better decisions regarding 
additional costly testing
(e.g., who to interview or to recruit for a try-out period)
and consequent actions (e.g., who to hire).
But conventional learning methods optimize for accuracy on the distribution of input data, i.e., the train-time population of candidates;
this overlooks the important fact that \emph{who will apply after model deployment}---and who will not---often depends on the screening rule itself.
\squeeze


In this work we study classification of strategic agents that choose whether to apply or not in response to the learned classifier.
Strategic candidates apply only if the expected utility from passing screening outweighs associated costs;
thus, application choices derive from beliefs regarding classification outcomes.
Since these choices in aggregate determine the test-time distribution,
learning becomes susceptible to \emph{self-selection}---%
namely selection that is carried out by the agents which predictions target.
Our goal in this paper is to study learning under such self-selective behavior,
which we believe is prevalent in many application domains.
We seek to:
(i) establish the ramifications of self-selection on conventional learning methods;
(ii) propose a strategically robust method that is accurate on the self-selective distribution it induces;
(iii) study the power of such methods to influence choices and shape the applicant population;
and (iv) propose means for regulating and mitigating potential ill effects.
\squeeze

Our setting considers a firm which trains a classifier to be used for screening, where applicants who pass screening then partake in an accurate but costly test (e.g., trial period) that determines final outcomes (e.g., hiring).
Candidates would like to pass the test, but also to avoid unnecessary testing costs;
the challenge for them is that they do not know a-priori whether or not they will pass screening,
making their decisions regarding application inherently uncertain.
To cope with this, candidates can make use of relevant statistics regarding their chances of being hired.
We imagine these as being made public either by 
a third party (e.g., auditor, media outlet),
or by the firm itself,
e.g. due to regulation on transparency
\citep{matthews2023addressing}
or as a service to prospective candidates.%
\footnote{This is similar in spirit to e.g. credit calculators, that based on partial information provide an estimated range of likely credit scores.\squeeze}
The statistics we consider rely on a subset of (categorical) features describing candidates that provide semi-individualized, group-level information, useful to them for making informed application decisions.
Since the choice of classifier determines the reported statistics,
these become the interface through which learning influences applications.
This process is illustrated in Figure~\ref{fig:illust}.
\squeeze

The goal of learning in our setting is to train a classifier that will be accurate on the induced applicant distribution---as determined by the classifier, indirectly through how it shapes self-selection.
We study how learning approaches of increasing strategic sophistication affect, and are affected by, the process of self-selection.
We begin by showing that whereas learning optimizes for accuracy, candidates benefit from the classifier's \emph{precision}, which governs their decisions regarding application.
A classifier's performance on the induced (test-time) distribution therefore depends on how it balances accuracy and precision.
This also means that a \emph{strategic} learner can maximize induced accuracy
by carefully controlling its precision for different candidates
as a means for shaping the population of eventual applicants.
Our results show that this, coupled with the firm's informational advantage, provides it with much power:
under mild conditions, learning can fully determine for each group in the population whether its members will apply, or not.

To restrict this power, we propose to enforce a certain independence criterion, which draws connections to the literature on fairness.
Our main result here is that this ensures that applications adhere to a natural, classifier-independent `ordering', which relies only on the innate group-level base rate.
We show how this can allow a social planner to implement affirmative action policies using targeted subsidies.
\squeeze


We then switch gears and turn to proposing a practical method for learning under strategic self selection.
Our method is differentiable and so can be optimized using gradient methods.
Our first step is to model self-selection in the objective using per-example weights, where $w_i=1$ if candidate $i$ applies, and $w_i=0$ if she does not;
importantly, these weights depend on the learned classifier.
We then show how weights can be effectively `smoothed',
so that gradients can be passed through application decisions.
The challenge is that applications depend on precision,
which in turn depends on the predictions of the classifier that is being optimized.
For this we propose a differentiable proxy for (conditional) precision,
and provide an effective implementation.
We conclude with an empirical demonstration of our approach 
in a semi-synthetic experimental setting that uses real data and simulated self-selective behavior.
Code is publicly available at
\url{https://github.com/Ysommer/GKSC-ICML}.
\squeeze

\begin{figure}[t!]
\centering
\includegraphics[width=\columnwidth]{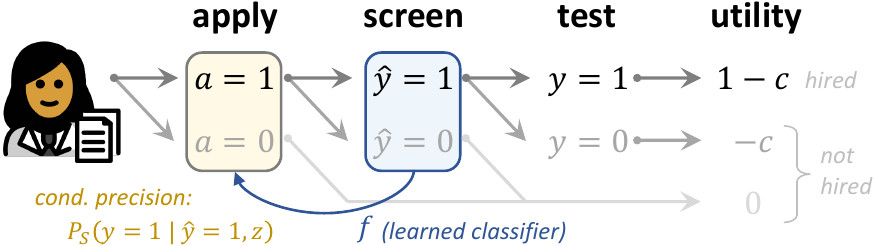}
\caption{\textbf{The application process}. 
Candidates who apply must first pass a screening classifier;
if successful, they advance to take a costly qualifying test.
Candidates are strategic, and apply only if it is cost-effective.
Since their likelihood of passing screening depends on the classifier
(through its conditional precision on past data),
learning has the power to shape the composition of the applicant population.
\squeeze
}
\label{fig:illust}
\end{figure}



\subsection{Related work} \label{sec:related}

\paragraph{Strategic classification.}
Our work is tightly connected to the growing literature on \emph{strategic classification} 
\citep{hardt2016strategic,bruckner2012static},
in which learning must cope with agents that can strategically modify their features (at a cost) in order to obtain preferred predictive outcomes.
\extended{Alongside works that study practical \citep{levanon2021strategic} and theoretical \citep{zhang2021incentive,sundaram2021pac} aspects of learning,}
There is an ongoing effort to extend and generalize beyond the original problem setting; examples include support for
richer models of user behavior
\citep{jagadeesan2021alternative,sundaram2021pac,levanon2022generalized,eilat2023strategic},
relaxing informational assumptions
\citep{ghalme2021strategic,bechavod2022information,barsotti2022transparency,lin2023plug,shao2023strategic,lechner2023strategic,harris2023strategic,rosenfeld2023one},
and introducing causal elements \citep{miller2020strategic,chen2023linear,horowitz2022causal,mendler2022anticipating}.
These works, as well as the large majority of others in the field, focus on feature modification as the action that users can take
(one notable exception is \citet{krishnaswamy2021classification}, who allow users to withhold certain features).
In contrast, our work extends the literature by considering a drastically different type of action---namely the initial choice of users regarding whether to participate or not.
\squeeze

\extended{TODO: cite 'Strategic Usage in a Multi-Learner Setting' (concurrent) - user action is choosing amongst alternatives, so sort of self-selection}




\paragraph{Screening, selection, and self-selection.}
The study of self-selection has a significant history in economics;
some recent works that are relevant to our context include 
\citet{Lagziel} and \citet{Lagziel2019} who analyze signal distributions in filtering mechanisms and identify conditions leading to inefficiencies from excessive filtration steps in selection processes;
\citet{Carroll2017} who focuses on screening in principal-agent models; and \citet{courtySequentialScreening2000} who investigates dynamic pricing as a tool to screen consumers with low willingness to pay.
Most related to ours is \citet{korenGatekeeperEffectImplications2023},
which establishes the connection between hiring and self-selection, as mediated by the quality of screening.
In machine learning, several recent works study the use of classifiers for screening
\citep{wang2022improving},
also for strategic agents \citep{cohen2023sequential,beyhaghi2023screening},
and with connections to fairness
\citep{khalili2021fair,blum2022multi,okati2023within}.
Other works study learning in settings with self-selection,
although in different contexts and with differing goals.
\citet{zhang2021classification} consider a sequential screening setting where applicants can decide when (and if) to quit, and show how self-selection can be exploited.
\citet{ben-porat2022modeling} model user attrition in recommendation systems as a bandit problem with `departing' arms.
\citet{cherapanamjeri2023makes} give algorithms for endogenous self-selection that controls realized labels (rather than participation).
These works have a strong game-theoretical emphasis;
in contrast, our focus is primarily on learning aspects.
\extended{Additionally, almost all of the above works consider sequential screening processes, whereas we consider a single deployment.}
\squeeze



\section{Problem setup} \label{sec:setup}

Consider a firm interested in training a classifier to be used for screening job applicants.
Prospective candidates are represented by features $x$
and a binary label $y \in \{0,1\}$ indicating whether the candidate is qualified or not.
We assume that $x$ includes at least some categorical features,
but allow the other features to be of any type or modality (e.g., vectors, images, text).
Candidates are assumed to be sampled iid from some unknown joint distribution as $(x,y) \sim p$.
Given a sample set $\smplst =\{(x_i, y_i) \}_{i=1}^m \sim p^m$,
the firm seeks to train a classifier $\yhat = f(x)$
to accurately predict labels $y$ for unseen future candidates $x$.
Typically we will have that $f(x)=\one{\score(x)>0}$ where $\score$ is a learned score function.
Once $f$ is obtained, it is used by the firm as a `gatekeeper' for screening: 
any candidate predicted to be qualified (i.e., has $\yhat=1$) is invited to partake in an accurate (but costly) qualification test or trial period which reveals her true $y$.%
\footnote{This is similar to the screen-then-test setup of \citet{blum2022multi}.}
The firm then hires any candidate deemed qualified, i.e., has $y=1$.
\squeeze


\paragraph{Strategic application.}
Candidates would like to be hired,
but also to avoid incurring the possibly unnecessary costs of potential testing.
Assuming w.l.o.g. that candidates gain unit utility from being hired (which occurs iff $y=1$),
let $\costu \in [0,1]$ be the cost candidates incur when taking the test.
\todo{make sure this is actually wlog!}
We assume candidates are \emph{strategic}, and hence make informed decisions regarding whether to apply,
denoted $a \in \{0,1\}$.
These are made on the basis of information regarding the screening process,
as it depends on the learned $f$.
Since testing takes place only if a candidate applies \emph{and} passes screening (i.e., obtains $\yhat=1$),
utility is given by:
\begin{equation}
\label{eq:utility}
u(a) =  a \cdot \yhat \cdot (y-\costu)   
\end{equation}
Candidates would like to apply only if this admits positive utility, i.e., if $u(1) \ge 0$ (note $u(0)=0$).
The firm, however, does not provide pre-application access to individual predictions $\yhat$---%
i.e., candidates cannot know with certainty whether they will pass screening or not.
This, coupled with candidates not knowing their true $y$, means that 
$u(a)$ cannot be computed (nor optimized) exactly.
To cope with this uncertainty, we model candidates as 
rational decision-makers, who choose to apply iff this maximizes their expected utility:
\squeeze
\begin{equation}
\label{eq:rational}
a^* = \argmax\nolimits_{a \in \{0,1\}} \expect{\ptilde(y,\yhat \mid x)}{u(a)}
\end{equation}
where $\ptilde(y,\yhat \mid x)$ encodes their beliefs regarding the joint uncertainty in $y$ and $\yhat$, conditional on $x$ (which they know).
We refer to candidates who select to apply via $a^*=1$ as \emph{applicants}.
We next discuss what constitutes these beliefs.
\squeeze

\paragraph{Decision-making under uncertainty.}
To facilitate informed decision-making,
we assume that the firm publishes coarse aggregate statistics concerning $y$ and $\yhat$, which candidates then use to form beliefs $\ptilde$.
In particular, let $\grp \subset x$ be a subset of (categorical) features of size $k$,
and denote by $K$ the number of distinct values $\grp$ can take.%
\footnote{
Group variables $\grp$ can correspond to sensitive or protected variables, but are not necessarily such---rather, we think of them simply as the set of variables for which the firm chooses to (or must, e.g. due to regulation) report conditional statistics.
}
Then we assume the system makes public
the \emph{conditional precision} metrics:
\begin{equation}
\prc_\grp = P_\smplst(y=1 \mid \yhat=1, \grp) \quad \forall \grp
\end{equation}
where $P_\smplst$ is the empirical distribution over the training set $\smplst$.
Note that $\prc_\grp$ depends on the classifier $f$ through the conditioning on (positive) predictions $\yhat=f(x)$.
\squeeze

Precision provides candidates a rough estimate of their likelihood of being hired, given that they pass the screening phase.
By partitioning all candidates into $K$ `groups', as determined by $\grp$,
candidates can obtain partially-individualized group-level information by querying $\prc_\grp$.\footnote{
One reason for having variables $x \setminus \grp$ that are not conditioned on is that they materialize only after application:
for example, in the academic job market, whether a submitted paper is accepted or not, or the contents of a recommendation letter.
}
Interestingly, precision turns out to be sufficient for decision-making.
\begin{proposition}\label{prop:utility_simple_form}
Given a classifier $f$, the utility-maximizing application rule in Eq.~\eqref{eq:rational}
admits the following simple form:
\begin{equation}
\label{eq:decision_rule-precision}
a^* = \one{\prc_\grp \ge \costu}
\end{equation}
\end{proposition}
We defer all proofs to Appendix~\ref{appx:proofs}.
Eq.~\eqref{eq:decision_rule-precision} holds under the mild condition that $P_S(\yhat=1 \mid \grp)>0$, i.e., as long as in each $\grp$, not all candidates are classified as negative. %
Importantly, once beliefs are shaped by $\prc_\grp$,
who applies---and who does not---becomes dependent on the learned $f$.%
\footnote{
Note that precision accounts for a strict subset of the information in $\ptilde(y,\yhat \,|\, \grp)$,
which is generally required for $a^*$ in Eq.~\eqref{eq:rational}.
}
This idea is illustrated in Fig.~\ref{fig:prc}.
When needed, we will use $a^*_\grp$ to denote applications under $\grp$.
Note Eq.~\eqref{eq:decision_rule-precision} implies that $a^*$ does not depend on screening outcomes.
\squeeze


\todo{need to say something about uncertainty remaining in $x \setminus \grp$?}

\begin{figure}[t!]
\centering
\includegraphics[width=\columnwidth]{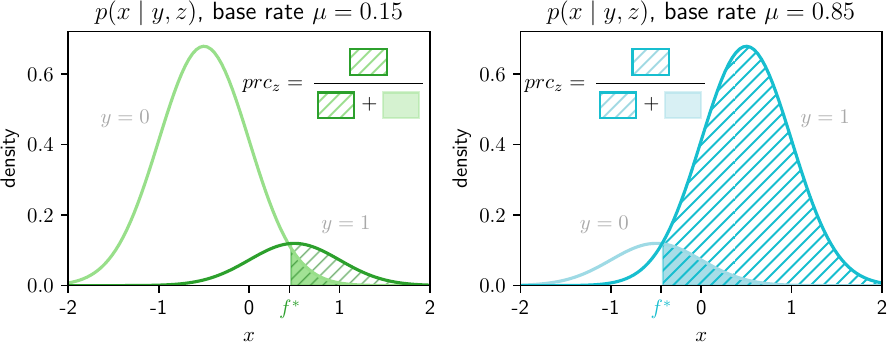}
\caption{
$\prc_z$ under optimal $f^*$ for high and low base rates.
\squeeze
}
\label{fig:prc}
\end{figure}

\paragraph{Learning under self-selection.}
Since the goal of screening is to reduce the load on testing,
screening needs to be accurate on the pool of \emph{applicants}---not on the entire population of candidates.
For a given classifier $f$, denote by $p^f$ its induced distribution over applicants, defined as:
\begin{equation}
p^f(x,y) = 
\begin{cases}
    \frac{1}{A} p(x,y) & \text{ if } a^*=1 \\
    0 & \text{otherwise}
\end{cases}
\end{equation}
where $A{=}\int \!a^* p(x,y) \d x \d y$ is the normalizing constant.
In other words, the probability of sampling a candidate 
from $p^f$
remains proportional to $p$ (with respect to all applicants) if 
the applicant applies,
and zero otherwise.
We refer to $p^f$ as the \emph{self-selective distribution} induced by $f$
(Fig.~\ref{fig:dists}).
Given this, the goal in learning is to minimize predictive error on this induced applicant distribution: 
\squeeze
\begin{equation}
\label{eq:objective_induced}
\argmin\nolimits_{f \in F} \, \expect{p^f(x,y)}{\one{y \neq f(x)}}
\end{equation}
where $F$ is some chosen model class (e.g., linear classifiers or neural networks).
\todo{maybe say we focus on linear - depending on final analytic results, and what we end up doing in the experiments}
Note that in Eq.~\eqref{eq:objective_induced}, both the loss \emph{and} the distribution in the expectation depend on the optimized $f$,
making it an instance of learning under
\emph{decision-dependent distribution shift} \citep{drusvyatskiy2022stochastic}.
Our goals will be to devise a method for optimizing Eq.~\eqref{eq:objective_induced} effectively, and to study the effects of different learning approaches on application outcomes.
\squeeze


\subsection{Preliminaries} \label{sec:prelim}
Before turning to our main results,
we begin with some basic analysis which sheds light on important aspects of our setup.
\squeeze


\paragraph{The role of precision.}
Since the decisions of candidates are based on (conditional) precision,
a key question is whether higher precision is beneficial for them.
Generally, the answer is yes---this is since increased precision can enable $a^*=1$, which implies that the utility gained is positive (vs. $u=0$ for $a^*=0$).
However, this connection is more nuanced,
and is made precise by the following result:
\squeeze
\begin{proposition} \label{prop:util_mono_in_prc}
For any group $\grp$,
its expected utility
is monotonically increasing in conditional precision $\prc_\grp$,
as long as its positive prediction rate $P_\smplst(\yhat=1 \mid \grp)$ is kept fixed.   
\squeeze
\end{proposition}
Nonetheless, and perhaps surprisingly, higher precision does not always entail better outcomes for applicants:
\begin{proposition} \label{prop:util_not_mono_in_prc}
There exist classifiers $f_1,f_2$ where $f_1$ has higher precision, but $f_2$ entails higher utility for applicants.
\end{proposition}
The proof is constructive, and relies on a simple contingency table.
Thus, the choice of classifier not only determines who applies,
but also the potential benefit of applying.

\paragraph{Incentive alignment.}
Since learning optimizes for accuracy,
but users generally seek higher precision,
it is natural to ask how these two incentives relate.
While not precisely aligned, our next result
shows that they are tightly connected.
For a group $\grp$,
let $\base_\grp = P_\smplst(y=1 \mid \grp)$ denote its \emph{base rate} (which does not depend on $f$),
and $\acc_\grp = P_\smplst(y=\yhat \mid \grp)$ denote its empirical accuracy (which does).
\squeeze
\begin{proposition}
\label{prop:acc_thresh_apply-or-not}
Fix $\costu$, and consider some group $\grp$. 
\begin{itemize}[leftmargin=1em,topsep=0em,itemsep=0.3em]
\item 
If $\acc_\grp < 1-\base_\grp \max\{1,\frac{1}{c}-1\}$,
then necessarily $a^*_\grp=0$.
\item 
If $\acc_\grp \ge 1-\base_\grp \min\{1,\frac{1}{c}-1\}$,
then necessarily $a^*_\grp=1$. %
\end{itemize}
\end{proposition}
Prop.~\ref{prop:acc_thresh_apply-or-not}
shows that excessively low accuracy can impede application,
and that high enough accuracy can enable it
(see Fig.~\ref{fig:acc_of_c} in Appx.~\ref{appx:proofs} for an illustration of these relations).
Thus, disparity across group accuracies $\acc_\grp$ can translate to disparity in applications $a^*_\grp$.
One worrying implication is that the firm, by controlling accuracy,
can indirectly control for which groups applying is cost-effective.
Fortunately, there are regimes in which learning is devoid such power.
\squeeze
\begin{corollary}
\label{cor:unconstrained_apply}
Fix $\acc_\grp$.
If neither conditions from Prop.~\ref{prop:acc_thresh_apply-or-not} hold,
then $a^*_\grp$ is unconstrained: there exist $f_1,f_2$ that both attain $\acc_\grp$, but $a^*_\grp=0$ under $f_1$, and $a^*_\grp=1$ under $f_2$.
\end{corollary}
Prop.~\ref{prop:acc_thresh_apply-or-not} also reveals the roles of $\base_\grp$ and $\costu$ as mediating factors.
On the one hand, for fixed $\costu$, if some groups have low base rates $\base_\grp$, then this becomes an innate obstacle to equitable application.
On the other hand, if costs can be reduced for those groups,
then this disparatiy can be corrected---%
motivating our discussion on targeted subsidies in Sec.~\ref{sec:subsidies}.

\begin{figure}[t!]
\centering
\includegraphics[width=\columnwidth]{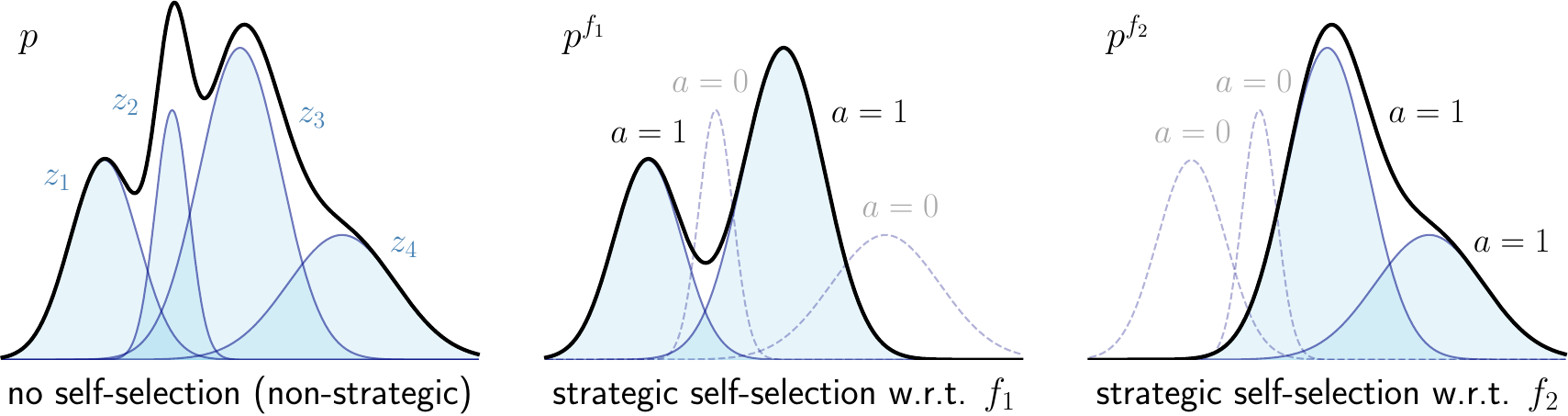}
\caption{\textbf{Self-selective distributions}. 
Absent strategic behavior, all groups in the population are assumed to participate (left).
But when applications depend on the learned classifier (here, $f_1$ vs. $f_2$), self-selection shapes the target distribution that the classifier will face (center vs. right).
}
\label{fig:dists}
\end{figure}

\section{Analysis} \label{sec:analysis}

In this section we consider the implications and possible outcomes of learning in the face of model-induced self-selection.
We begin with a simple form of learning,
and then proceed to consider increasingly more informed approaches.

\subsection{\Naive\ non-strategic learning} \label{sec:naive}
The constructive relation between accuracy and precision discussed in Sec.~\ref{sec:prelim} suggests that perhaps it may suffice to
employ a conventional learning approach that {\naive}ly maximizes predictive accuracy on the original distribution $p$,
despite this being oblivious to self-selection.
This may also seem to favor users:
notice that if we discard $a$,
then user utility, given by $\expect{p}{\yhat(y-c)}$, 
correlates with accuracy, which can be written as $\expect{p}{(2\yhat-1)(2y-1))}$,
since one is obtained from the other via linear transformations of $y$ and $\yhat$.
\squeeze

The crux, however, is that once $a$ \emph{is} accounted for,
this global correlation pattern breaks,
as does the relation between per-group accuracy and precision.
The result is a disconnect between the accuracy the model is assumed to obtain (w.r.t. $p$),
and its actual accuracy at test time (i.e., on the induced $p^f$).
This is demonstrated through a constructive example in Appx.~\ref{appx:naive}:
let $K=2$, and for each $\grp \in \{1,2\}$,
define $p(x,y|\grp)$ to be composed of two per-class Gaussians over $x$,
each centered at their corresponding $y$.
Then by varying the base rate $\mu_2$ (and keeping $\mu_1$ fixed),
where by varying a single $\base_\grp$,
we are able to generate arbitrary outcomes, where
assumed (i.e., standard) accuracy is either:
worse than induced accuracy;
appears to be better;
or remains the same.
\squeeze

\todo{if example is deferred to extended, give a sentence on possible outcomes or underspecification}

\subsection{Semi-strategic learning: varying the threshold} \label{sec:semi-strat}
Since precision determines applications, a slightly more sophisticated ``semi-stretegic'' approach
is to take a pre-trained `\naive' model and then tune its precision strategically to maximize induced accuracy.
For threshold classifiers $f_{\score,\thresh}(x)=\one{\score(x)>\thresh}$
with score function $\score$ and threshold $\thresh$,
here we build on the common practice of first training $f$ for accuracy, 
and then varying
$\thresh$ (with $\score$ fixed) to control precision. 
\squeeze

\paragraph{Monotonicity.}
Generally, precision is expected to increase with $\thresh$,
but this does not occur monotonically.%
\footnote{Intuitively, the reason is that $\thresh$, which determines $\yhat$,
appears in the denominator: $\prc = \frac{P(y=1,\yhat=1)}{P(\yhat=1)}$.
Recall, however, is monotone.}
To simplify our analysis, we present a condition which guarantees monotonicity;
this will ensure that when $\thresh$ is increased,
applications can transition from $a=0$ to $a=1$ at most once.
\squeeze

\begin{definition}[Calibrated score function] \label{def:calibration}
Let $\score$ be a score function such that
$p(\score(x),y)$ is a well-defined density, and
for which $\score$ has full support on $[\alpha,\beta]$.
We say $\score$ is a \emph{calibrated score function} w.r.t. $p$ if
for all $\tau \in [\alpha,\beta)$:
\[
P(y=1 \mid \score(x)>\thresh) \geq P(y=1 \mid \score(x)=\thresh)
\]
\end{definition}

Intuitively, $\score$ is calibrated if it captures the `direction' in which $P(y=1 | x)$ increases.%
\footnote{A similar definition appears in \citet{okati2023within}
in the context of fairness, but is stronger (Appx.~\ref{appx:within_group_monotone}) 
and used for different ends.\squeeze}
And while optimizing $f$ for accuracy does not guarantee calibration,
it is a desirable property which we can hope will emerge (perhaps approximately).
Note score function calibration is a weaker condition (and is implied by) standard calibration
(see Appx.~\ref{appx:calibration}).
\extended{
\begin{lemma} \label{lem:if_calibrated_then_calibrated_score}
Let $\phat(y=1|x)$ be a calibrated probabilistic classifier,
i.e., satisfies $P(y=1 \mid \phat=\thresh) = \thresh$.
Then $\score=\phat$ is a calibrated score function (with range $[0,1]$).
\end{lemma}
}
Our next result links calibration and monotonicity.
\squeeze
\begin{lemma} \label{lem:calibrated_score_function}
Let $\score$ be a score function.
Then $\score$ is calibrated if and only if the precision of
its corresponding $f_{\score,\thresh}$,
namely $P(y=1 \mid f_{\score,\thresh}(x) = 1)$,
is monotonically increasing in $\thresh$.
\end{lemma}
We say a score function is calibrated for group $\grp$ if it is calibrated w.r.t $p(\cdotp|\grp)$.
This ensures that $\prc_\grp$ is monotonic,
which in turn 
determines how varying $\thresh$ affects applications:
\squeeze
\begin{corollary} \label{cor:calib=>a_is_step_func}
Consider some group $\grp$. 
If $\score$ is calibrated w.r.t. $\grp$,
then $a^*_\grp$ is either a step function in $\thresh$, or is constant.
\end{corollary}


\todo{merge footnotes? link okati to multi-calib?}

\begin{figure}[t!]
\centering
\includegraphics[width=0.95\columnwidth]{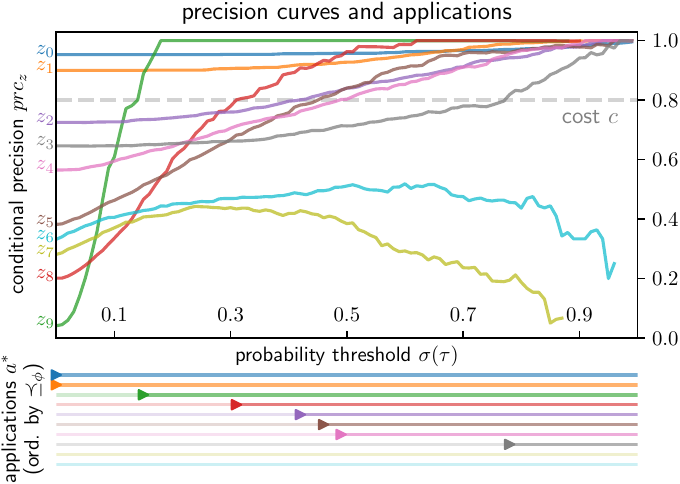}
\caption{
\textbf{Precision and application.}
For a fixed score function $\score$ and varying threshold $\thresh$,
different groups ($K=10$, colored lines) exhibit different precision curves.
Though $\score$ is not mutually calibrated, most curves are roughly monotone and cross the cost $\costu$ (dashed line) at most once.
This induces an ordering $\preceq_\score$ over applications $a^*$ (lower plot).
A semi-strategic learner can affect $a^*$ only by thresholding on $\preceq_\score$.
\squeeze
\todo{add extended comment on $\base$, how green starts low bu climbs fast, etc}
\todo{add similar plot but with $\yhat \perp \grp$ in appendix?}
}
\label{fig:precision_curves+ordering}
\end{figure}

\paragraph{Thresholds and ordering.}
We say that $\score$ is \emph{mutually calibrated} if it is calibrated 
for all $\grp$.%
\footnote{Works discussing standard calibration across groups includes in \citet{pleiss2017fairness} in the context of fairness
and in \citet{wald2021calibration} in the context of out-of-distribution generalization.}
When this holds,
Cor.~\ref{cor:calib=>a_is_step_func} implies that 
$\score$ induces an \emph{ordering} over groups,
and $\thresh$ serves to `threshold' applications w.r.t. that ordering.
Let $\costu_\grp$ be the minimal $\thresh$ s.t. $\prc_\grp \ge \costu$,
and define $\grp \preceq_\score \grp'$ iff $\costu_\grp \le \costu_{\grp'}$.
\squeeze
\begin{corollary}
\label{cor:score_ordering}
Let $\score$ be mutually calibrated, and consider some $\grp,\grp'$
for which $\grp \preceq_\score \grp'$. Then for any $\thresh$ under which
candidates from $\grp'$ apply, candidates from $\grp$ apply as well.
\end{corollary}

This idea is exemplified in Fig.~\ref{fig:precision_curves+ordering}, which plots per-group precisions $\prc_\grp$ as a function of $\thresh$
for a standard logistic regression model trained on simple synthetic data
($K=10$, class-conditional Gaussians $p(x|y,\grp)$ per group $\grp$; details in Appendix~\ref{appx:exp_details_synth}).
As can be seen, mutual calibration does not hold,
but nonetheless most $\prc_\grp$ are generally increasing, and cross
$\costu$ at most once, which enables the ordering $\preceq_\score$.
\squeeze

\extended{$z_9$ has high slope. $z_5$ and $z_6$ start similar, but diverge. $z_0$ and $z_1$ have $\base_\grp>\costu$ (cor).}

This points to an inherent limitation of semi-strategic learning:
once $\score$ is fixed, applications that result from tuning $\thresh$
must comply with the constraints imposed by $\preceq_\score$.
On the one hand, this provides an innate restraint on the firm's power to determine who applies and who does not.
On the other hand, it impedes the firm's ability to obtain its predictive goals:
Appendix~\ref{appx:synth_semi-strat} shows an example with $K=2$
where any reasonable classifier induces an ordering $\grp_1 \preceq \grp_2$,
but optimality w.r.t. induced accuracy necessitates $\grp_1 \succ \grp_2$.

\paragraph{Base rates and applications.}
Notice that for the minimal $\thresh$, it holds that $\prc_\grp = \base_\grp$;
this is since \emph{all} predictions are positive, and therefore conditioning on $\yhat = 1$ is vacuous.
Coupled with calibration, this has concrete implications:
\squeeze
\begin{corollary} \label{cor:base>c_always_applies}
Let $\grp$ with $\base_\grp \geq \costu$.
If $\score$ is calibrated w.r.t. $\grp$, 
then $\grp$ always applies, i.e., has $a^*_\grp=1$ for any value of $\thresh$.
\end{corollary}
\todo{before taking out - remember this is used in later results!}
One consequence is that a semi-strategic firm wishing to favor `quality' groups
(i.e., having high $\base_\grp$) can easily do so by revoking screening altogether
(via $\thresh = -\infty$),
so that $a^*=\one{\base_\grp \geq \costu}$:
this creates an appearance of equal opportunity,
but in effect discriminates against low-$\base_\grp$ candidates
by exploiting their need to overcome a larger cost gap.
\squeeze


Nonetheless, for general $\thresh$, the restraining effect of $\base_\grp$ can weaken.
Note how in Fig.~\ref{fig:precision_curves+ordering} curves differ in their intercepts (equal to $\base_\grp$), but also in their `slopes'.
Slopes are given by the part of $\prc_\grp$ that does not depend on $\base_\grp$---using Bayes, we can write:%
\footnote{This slope term is sometimes referred to as `normalized' recall.}
\begin{equation}
\prc_\grp = \frac{P(\yhat=1 \mid y=1, \grp)}{P(\yhat = 1 \mid \grp)} \base_\grp  
\end{equation}
This shows how low-$\base_\grp$ groups can apply `before' (i.e., under lower $\thresh$)
other groups with higher base rates.
For example, in Fig.~\ref{fig:precision_curves+ordering}, note how $\grp_9$ has low base rate, but its precision curve rises quickly, and it applies early in $\preceq_\score$.
Another example are $\grp_5$ and $\grp_6$ who have similar base rates and therefore begin similarly, but end up very differently in the ordering.


\subsection{Strategic learning: anticipating self-selection} \label{sec:analysis_strat}
If the firm is aware of self-selection, then it should benefit from encoding this directly into the learning objective.
A natural objective for a strategic firm is therefore:
\begin{equation}
\label{eq:objective_induced_empirical}
\argmin_{f\in F} \frac{1}{m^f} \sum\nolimits_i a^*_i  \loss(y_i,f(x_i))
\end{equation}
which is the empirical analog of Eq.~\eqref{eq:objective_induced}.
Here, $\ell$ is some proxy loss (e.g., log-loss or hinge loss),
$a^*_i$ is the application decision of candidate $i$ (recalling this depends on $f$),
and $\quad m^f = \sum_i a^*_i$ is the total number of applicants under $f$.%
\footnote{Technically, Eq.~\eqref{eq:objective_induced_empirical} is ill-defined if $a^*_i=0$ for all $i$. To circumvent this outcome, we can define $\frac{0}{0}=\infty$; in practice, we implement this using an additive penalty term (see end of Sec.~\ref{sec:method}).}


\paragraph{Strategic power.}
Once learning accounts for self-selection,
it can utilize its informational advantage over candidates:
the learner knows $y$ and can compute $\yhat$ for any $x \in \smplst$,
whereas candidates only have access to $\prc_\grp$.
A key question is: %
what are the reaches of this power? 
Our next result shows that the firm's control can be quite extensive.
\squeeze
\begin{proposition}
\label{prop:strategic_single_group_applies}
Let $F$ be any class of functions with group-specific `offset' terms $v_\grp$,
i.e., are of the form $f(x)=g(x)+v^\top \grp$.
If all groups have $\base_\grp<\costu$,
then the optimal $f$ for Eq.~\eqref{eq:objective_induced_empirical}
is such that only a single group applies.
Otherwise, $f$ is such that
of all groups with $\base_\grp<\costu$,
only at most one group applies.
In both cases, this group is that which attains the highest accuracy had it been trained on separately. 
\squeeze
\end{proposition}
\todo{does this assume calibration? if yes, then say this}
The proof is constructive, giving an algorithm that could theoretically obtain the optimal $f$.
Intuitively, Prop.~\ref{prop:strategic_single_group_applies} holds
since linear models $w^\top x + b$ have coefficients $w_\grp$ associated with each $\grp$ (encoded as 1-hot);
thus, $f$ can be constructed to ensure that only a specific target group $\grp$ applies by 
setting $w_\grp = 0$;
setting $w_{\grp'}=-\infty$ for all other $\grp' \neq \grp$;
and training the remaining coefficents $w_{x \setminus \grp}$ and $b$
to predict well on examples from $\grp$ in $\smplst$, subject to $\prc_\grp \geq \costu$.

\todo{careful - reviewer could ask why this doesn't happen empirically. either say the algo is theoretical, or requires oracle access to the minimizer of the standard loss.}

Prop.~\ref{prop:strategic_single_group_applies} asserts that
even with simple models, strategic learning has the capacity to 
designate a single group---that which enables the highest accuracy---%
and block all others candidates from applying.
But this power extends further:
if the firm has in mind goals other than accuracy,
it can essentially
shape applications, through self-selection, as it sees fit.
\begin{corollary} \label{cor:specific_groups}
Let $F$ include functions with group-specific offsets $v_\grp$ as before,
and consider some $f \in F$.
Then for any set of groups $A=\{\grp \,:\, \base_\grp<\costu\} \subseteq [K]$,
it is possible to construct an $f'$ which agrees with $f$ on all $x$ with $\grp \notin A$,
but prevents application from all candidates in any $\grp \in A$.
\squeeze
\end{corollary}

\paragraph{Taming strategic learning.}
The power to determine application outcomes derives from the ability to influence each group individually.
To restrict this power, one idea is to simply remove $\grp$ from the set of features.
The caveat in this approach is that even if $\grp$ is not used explicitly,
if the remaining features $x \setminus \grp$ are informative of $\grp$,
then $\grp$ could be exploited implicitly.
To ensure that learning is entirely agnostic to group memberships,
we propose to augment the objective in Eq.~\eqref{eq:objective_induced}
with a constraint enforcing independence:
\squeeze
\begin{equation}
\label{eq:fairness_constraint}
\argmin\nolimits_{f \in F} \, \expect{p^f}{\one{y \neq f(x)}}
\quad\,\,\, \text{s.t.} \quad
f(x) \perp \grp
\end{equation}
\todo{should/can we instead have $\psi(x) \perp \grp$ or $x \perp \grp$?}
\todo{switch to empirical objective?}
In the fairness literature, this constraint is known as \emph{statistical parity} \citep{dwork2012fairness}.
Note that independence must be enforced on the training distribution $p$,
rather than on the induced $p^f$ (i.e., after the fact);
otherwise, a strategic learner could `choose' to discard any $\grp$
for which satisfying Eq.~\eqref{eq:fairness_constraint} is either difficult
or hurts predictive performance.

\todo{important: indep considers only predictions - learning still makes use of z in a to anticipate applications. if we were to remove z entirely from the data, then we can't learn strategically, only naively. but then this is back to sec 3.1. this also happens to the `remove z' heuristic - its removed from f, not from a.}

Our next result shows that statistical parity limits the power of strategic learning in a particular way%
---by inducing an ordering over group applications.
Let $\preceq_\base$ be an ordering over groups by decreasing base rate,
namely $\grp \preceq_\base \grp'$ if $\base_\grp \ge \base_{\grp'}$.
\squeeze
\begin{proposition}
\label{prop:indep_order}
Let $f$, and assume $f(x) \perp \grp$.
For any $\grp,\grp'$ with $\grp \preceq_\base \grp'$,
if $\grp'$ applies under $f$, then $\grp$ also applies.
\end{proposition}
Thus, if $f$ satisfies statistical parity, then applications must comply with $\preceq_\base$;
in other words, $f$ can still determine who applies, but now only by thresholding
on $\base$.
Note this ordering is independent of $f$
(c.f. $\preceq_\score$ from Cor.~\ref{cor:score_ordering}, which is).
This follows from precision now admitting the following form:
\squeeze
\begin{equation}
\label{eq:prec_indep}
\prc_\grp = P_\smplst(y=1 \mid \yhat=1) \frac{P_\smplst(y=1 \mid \grp)}{P_\smplst(y=1)}
= \prc \frac{\mu_\grp}{\mu}
\end{equation}
where $\prc$ and $\mu$ are the global precision and base rate, respectively (see proof in Appendix~\ref{appx:proofs}).
Applications can therefore be rewritten as
$a^* = \one{\prc \geq c \cdot \frac{\base}{\base_\grp}}$,
and since $f$ only affects $\prc$, and does so globally,
high $\base_\grp$ can be interpreted as reducing `effective' per-group application costs.


\todo{add plot showing how lines move together both with different slopes?}

\subsection{Affirmative action} \label{sec:subsidies}


Consider a social planner who wishes to promote the application of some group $\grp$.
If statistical parity holds, then 
the social planner can grant \emph{targeted subsidies},
$\sbsdy_\grp \ge 0$, to reduce costs for candidates in $\grp$.
The decision rule becomes: 
\begin{equation}
\label{eq:application_targeted_subsidy}
a^*_\grp = \one{\prc_\grp \geq \costu_\grp}
\quad \text{where} \quad
\costu_\grp = \costu - \sbsdy_\grp
\end{equation}
Here,
per-group costs $\costu_\grp$ correct for low $\base_\grp$ (in units of $\base$) as:
\squeeze
\begin{equation}
\prc_\grp(\sbsdy_\grp) \coloneqq
\prc \frac{\base_\grp(\sbsdy_\grp)}{\base}, \qquad 
\base_\grp(\sbsdy_\grp) \coloneqq \base_\grp + s_\grp \base
\end{equation}
Subsidies can be used to `bump up' the target group in the ordering $\preceq_\base$ (Prop.~\ref{prop:indep_order}).
Nonetheless, a promoted group will apply only if its inclusion does not degrade potential accuracy, given its new position.
With calibration,
$\grp$ can be guaranteed to apply if $\sbsdy_\grp$ is sufficiently large 
so that $\base_\grp > c_\grp$ (Cor.~\ref{cor:base>c_always_applies}).
\squeeze

\todo{if learner controls subsidies, then can only help (if does not have to allocate all)}

\todo{
weaker requirement: ensure applications are independent on group *given base rate* - that is, group identity doesn't matter, only its base rate.
and this is correctable by subsidies - because they counter low base rate.}

\todo{what else can we say here? maybe result on incentives and/or alignment? (shapley?)}

\section{Method} \label{sec:method}
We now turn to describing an effective method for optimizing the strategic learning objective in Eq.~\eqref{eq:objective_induced_empirical}.
This objective differs from the standard ERM objective in that the classifier $f$ determines for each example not only its prediction, but also, through $a^*$, whether it should be `turned on'
or not.
The challenge is therefore to account for the dependence of application decisions $a^*$ on the classifier $f$ being optimized.
This is further complicated by the fact that $a^*$ is discrete. 

Our solution, which jointly addresses both problems, is to replace $a^*_i \in \{0,1\}$ with a continuous surrogate $\atilde_i \in [0,1]$ that is differentiable in the parameters of $f$. 
These, together with the normalizing factor,
are used to define differentiable per-example weights,
$w_i^f = \atilde_i/\mtilde^f$,
where $\mtilde_f=\sum_i \atilde_i$,
so that $\sum_i w_i^f =1$ always.
Our proposed objective is:
\begin{equation}
\label{eq:objective_induced_smoothed}
\argmin\nolimits_{f\in F} \sum\nolimits_i w_i^f  \loss(y_i,f(x_i)) + \lambda R(f)
\end{equation}
where $R$ is an (optional) regularization or penalty term.
We now describe how to effectively implement 
each component.

\paragraph{Differentiable applications.}
Recall that for a candidate $i$ with $\grp_i = \grp$, her application decision is $a^*_i = \one{\prc_\grp \geq \costu}$ (see Eq.~\eqref{eq:decision_rule-precision}).
A natural first step is to replace the indicator function with a smooth sigmoidal function $\sig$, so that:
\begin{equation}
\label{eq:application_smooth}
\atilde_i=\sig(\prc_\grp - \costu)
\end{equation}
Note however that the standard sigmoid is inappropriate, since its domain is the real line, whereas $\prc_\grp$ and $\costu$ are both in $[0,1]$.
A proper alternative should satisfy the following properties:
(i) have $\sig(-\costu) = 0$ for $\prc_\grp=0$,
and $\sig(1-\costu) = 1$ for $\prc_\grp=1$;
(ii) be indifferent at $\prc_\grp=\costu$, i.e., $\sig(0) = 0.5$;
(iii) include a temperature parameter $\tau$ s.t. $\lim_{\temp \to \infty} \sig_\temp = \mathds{1}$.
We propose the following sigmoid:
\squeeze
\begin{equation}
\label{eq:sigmoid}    
\sig_\temp(r;\costu) = 
\bigg(
1+\bigg(  
\frac{(r+\costu)(1-\costu)}{\costu(1-(r+\costu))}
\bigg)^{\!\!-\temp\,}
\bigg)^{\!\!-1}
\end{equation}
{where $\temp \ge 1$.
Eq.~\eqref{eq:sigmoid}
is differentiable and satisfies all three criteria,
as illustrated in Appx.~\ref{appx:sigmoid}.
Note how the domain is $[-\costu,1-\costu]$, which shifts with $\costu$.
In practice we add tolerance as $\costu+\varepsilon$ to safeguard against statistical discrepancies between train and test.
\squeeze


\begin{table*}[t!]
\centering
\caption{\textbf{Experimental results.}
For representative $\costu \in \{0.7,0.8\}$ and averaged over 10 random splits, results show:
induced accuracy ($\pm$stderr),
number of applying groups,
and the $r^2$ between the ideal $\preceq_\base$ and the actual ranking based on $\prc_\grp$.
\squeeze
}
\resizebox{\textwidth}{!}{
\begin{tabular}{lrccccccccccccccc}
  &   & \multicolumn{3}{c}{\textbf{\adult} ($\costu=0.7$)} &   & \multicolumn{3}{c}{\textbf{\adult} ($\costu=0.8$)} &   & \multicolumn{3}{c}{\textbf{\bank} ($\costu=0.7$)} &   & \multicolumn{3}{c}{\textbf{\bank} ($\costu=0.8$)} \\
\cmidrule{3-5}\cmidrule{7-9}\cmidrule{11-13}\cmidrule{15-17}  &   & ind. acc. & apply & rank $r^2$ &   & ind. acc. & apply & rank $r^2$ &   & ind. acc. & apply & rank $r^2$ &   & ind. acc. & apply & rank $r^2$ \\
\cmidrule{1-1}\cmidrule{3-5}\cmidrule{7-9}\cmidrule{11-13}\cmidrule{15-17}\naivemthd &   & 85.2\tiny{\,$\pm0.3$} & 3.0/4 & 0.219 &   & (87.8) & 0.1/4 & (0.219) &   & (80.9) & 5.4/10 & (0.066) &   & - & 0/10 & - \\
\semi &   & 87.4\tiny{\,$\pm0.6$} & 2.1/4 & 0.135 &   & 90.1\tiny{\,$\pm0.5$} & 1.2/4 & 0.300 &   & \boldmath{}\textbf{90.0\tiny{\,$\pm0.4$}}\unboldmath{} & 1.5/10 & 0.068 &   & 86.4\tiny{\,$\pm0.6$} & 2.4/10 & 0.100 \\
\stratx &   & \boldmath{}\textbf{91.1\tiny{\,$\pm0.5$}}\unboldmath{} & 1.1/4 & 0.076 &   & \boldmath{}\textbf{90.5\tiny{\,$\pm0.5$}}\unboldmath{} & 1.0/4 & 0.003 &   & \boldmath{}\textbf{90.1\tiny{\,$\pm0.4$}}\unboldmath{} & 1.8/10 & 0.177 &   & \boldmath{}\textbf{88.7\tiny{\,$\pm0.4$}}\unboldmath{} & 1.1/10 & 0.238 \\
\stratnoz &   & 86.0\tiny{\,$\pm0.4$} & 2.5/4 & 0.244 &   & 89.0\tiny{\,$\pm0.5$} & 1.0/4 & 0.029 &   & 87.9\tiny{\,$\pm0.4$} & 6.7/10 & 0.170 &   & 87.0\tiny{\,$\pm1.1$} & 1.5/10 & 0.121 \\
\stratindep &   & 86.5\tiny{\,$\pm0$} & 0.6/4 & \textbf{0.85} &   & 88.5\tiny{\,$\pm0$} & 0.4/4 & \textbf{0.537} &   & 87.3\tiny{\,$\pm0.6$} & 0.9/10 & \textbf{0.343} &   & 88.2\tiny{\,$\pm0.5$} & 1.3/10 & \textbf{0.361} \\
\cmidrule{1-1}\cmidrule{3-5}\cmidrule{7-9}\cmidrule{11-13}\cmidrule{15-17}\end{tabular}%
}%
\label{tbl:results}%
\end{table*}%

\paragraph{Differentiable precision proxy.}
For $\atilde_i$ in Eq.~\eqref{eq:application_smooth} to be differentiable, $\prc_\grp$ must also be differentiable.
Note that:
\begin{equation}
\label{eq:precision}    
\prc_\grp = P_\smplst(y=1 \mid \yhat=1, \grp) =
\frac{\sum_{j:\grp_j=\grp} y_j \, \yhat_j}{\sum_{j:\grp_j=\grp} \yhat_j}
\end{equation}
Thus,
it suffices to replace `hard' predictions $\yhat$ with `soft' probabilistic predictions $\ytilde \in [0,1]$. For the common case where $\yhat = \one{\score(x)>0}$, we can replace the indicator with the standard sigmoid
$\sigalt$,
e.g., $\sigalt(r) = (1+e^{-r})^{-1}$, giving: 
\squeeze
\begin{equation}
\label{eq:precision_proxy}    
\softprc_\grp = 
\frac{\sum_{j:\grp_j=\grp} y_j \, \ytilde_j}{\sum_{j:\grp_j=\grp} \ytilde_j},
\qquad
\ytilde = \sigalt(\score(x))
\end{equation}
While sound, this approach suffers from a subtle form of bias, which makes it ineffective for our purposes.
To see this, consider that the nominator sums only over positive examples (i.e., with $y_j=1$). For a reasonably accurate $f$ (which implies $y, \yhat$ are correlated), since $\ytilde_j \le \yhat_j$, we can expect the nominator of $\softprc$ to be consistently \emph{smaller} than that of $\prc$;
if the denominator is only mildly affected (which is reasonable to assume),
then $\softprc$ as a proxy becomes negatively biased.
More worrying is the fact that this bias becomes \emph{worse} as accuracy improves---a phenomena we've observed repeatedly in our empirical investigations.
Thus, it becomes harder, through this proxy, to enable application for high-accuracy groups---which goes precisely against our learning goals.
%
To remedy this, we propose to apply a corrective term: 
\begin{equation*}
\label{eq:precision_correction}    
\softprc_\grp = 
\frac{\sum_{j:\grp_j=\grp} y_j \, \ytilde_j}{\sum_{j:\grp_j=\grp} \ytilde_j - \bias}, \quad
\bias = \frac{1}{\costu}\sum\nolimits_i (y_i - \costu)(\yhat_i - \ytilde_i)
\end{equation*}
Appendix~\ref{appx:corrective_term} shows how $\bias$ serves to de-bias $\softprc$.
In practice, we clip $\softprc$ to be in $[0,1]$ so that it will be well-defined.

The downside is that the corrected $\softprc$ is no longer differentiable (since it now includes hard predictions $\yhat$).
Our solution is to fix at each epoch the values of $\yhat_i$ from the previous iteration, and update them after each gradient step.
Because in each step (and especially in later stages of training) only a few examples are likely to flip predictions, and since $B$ sums over all examples,
we expect this to only marginally affect the resulting gradient computations.

\paragraph{Ensuring application and well-defined precision.}
One issue with precision is that it is undefined if $\yhat=0$ for all examples.
Although this does not affect our continuous proxy,
we would still like to ensure that at test time the true precision is well-behaved.
Similarly, it is important that at least one group applies.
Our solution is regularize via:
\squeeze
\begin{equation} 
\label{eq:app&prec_penalty}
R_{\mathtt{app}}(\score;\smplst) =
- \frac{1}{K}\sum\nolimits_\grp \log \max_{i \in \grp} \ytilde_i 
- \log \max_{\grp} \atilde_\grp
\end{equation}
As we will see, one drawback of non-strategic approaches is that 
in some cases learning results in no applications at all.




\paragraph{Implementing independence.}
Enforcing statistical parity as independence constraints is generally hard,
but there are many approximate methods (e.g., \citet{agarwal2018reductions}).
In line with our general differentiable framework,
we opt for a simple approach and add to the objective the penalty:
\begin{equation}
\label{eq:regularizer_penalty}
R_{\perp}(\score;\smplst) = 
\frac{1}{K}\sum\nolimits_\grp \big(\expect{}{\ytilde=1 \,|\, \grp} - 
\expect{}{\ytilde=1}\big)^2
\end{equation}


\section{Experiments} \label{sec:experiments}
We now turn to our experimental analysis
based on real data and simulated self-selective behavior.
We use two public datasets:
(i) \adult\ and (ii) \bank,
both of which are publicly available,
commonly used for evaluation in the fairness literature \cite{le2022survey},
and appropriate for our setting.
As group variables, we use `race' for \texttt{adult} ($K=4$)
and `job' for \texttt{bank} ($K=10$).
We experiment with $\costu \in [0.65,0.85]$,
since this range includes all significant variation in application outcome.
Data is split 70-30 into train and test sets.
Results are averaged over 10 random splits and include standard errors.
Appendix~\ref{appx:exp_details} provides further details on data and preparation, methods, and optimization.
\todo{Appx extended and additional results}

\todo{bank - br=0 or =0.3? here and appx}


\paragraph{Methods.}
We compare between three general approaches, as discussed in Sec.~\ref{sec:analysis}:
(i) \naivemthd, which trains a classifier conventionally and is agnostic to self-selective behavior;
(ii) \semi, which implements the semi-strategic approach of first training a \naive\ classifier and then tuning its threshold strategically;
and (iii) \strat, which trains using our strategically-aware objective  (Sec.~\ref{sec:method}).
For the latter, we consider three variants:
(iii.a) \stratx, which uses all information in $x$;
(iii.b) \stratnoz, which discards group features $\grp$;
and (iii.c) \stratindep, which encourages statistical parity via regularization.
All methods are based on linear classifiers.
\squeeze

\paragraph{Optimization.}
For a clean comparison, all methods are based on the same core implementation (pytorch) and optimized in a similar fashion.%
\footnote{We also compared our implementation of \naivemthd\ to a vanilla sk-learn implementation, ensuring that they performed similarly.}
We used vanilla gradient descent with learning rate 0.1
and trained for a predetermined and fixed number of epochs.
Coefficients for $R_{\mathtt{app}}$ and $R_{\perp}$ (when used) were chosen to be small yet still ensure feasibility and (approximate) independence, respectively,
but overall performance was not very sensitive to chosen values.
\squeeze

\extended{\todo{replace with final public repo}}

\paragraph{Main results.}
Table~\ref{tbl:results} shows 
induced accuracies, application rates, and correlations between ranks based on $\prc_\grp$ and $\preceq_\base$
for representative $c =0.7, 0.8$
(full results in Appx.~\ref{appx:additional_exps}).
Note all accuracies are relatively high and in a narrow range;
this is since both datasets have inherently low base rates
($\base_{\mathtt{adult}}=0.24$, $\base_{\mathtt{bank}}=0.16$)
and so even a few percentage points in accuracy are significant.
The \naivemthd\ approach is clearly suboptimal;
moreover, in some cases it leads to no groups applying at all
(parentheses/hyphen indicate no applications in some/all splits).
Meanwhile, across all settings, \stratx\ obtains the highest induced accuracy. Notice how, as suggested by Prop.~\ref{prop:strategic_single_group_applies},
this follows from the number of applying groups being to be close to 1.
Interestingly, \semi\ does rather well,
albeit inconsistently;
more importantly, 
it does not support constraining for statistical parity,
which is necessary for enabling affirmative action.
\todo{is this because naive ordering is good to begin with? how to check?}
As expected, discarding features in \stratnoz\ and further adding constraints in \stratindep\ reduces their performance.
However,
in line with Prop.~\ref{prop:indep_order},
the large $r^2$ values suggest that in most cases \stratindep\ is able to approximately enforce the ordering $\preceq_\base$,
this effectively limiting the learner's power
(results for $\bank$ were highly variate for larger $c$).
\todo{inconsistent, highly varied (note stderrs) - sometimes works, sometimes fails, across seeds and inits. sometimes obtains fair by trivial all yhat=0 or =1. we conjecture balancing acc and fair on this dataset is challenging, also due to low base rate.}
\todo{explain rationale for using r2; say some br-s can be very close (this happens in adult), and so apx indep may swap them. r2 accounts for this, but on the other hand, will not get r2=1 even even order is correct wrt true br.}
In contrast, note how the excessive power of \stratx\ pushes rankings away from $\preceq_\base$.

\extended{semi may be good, but can't enforce fairness (in principle can constrain the line search, but because the strategic 'learning' part is 1D, its very non-expressive, so will likely pay a lot in accuracy - may be even get infeasible results)}

\begin{figure}[t!]
\centering
\includegraphics[width=0.85\columnwidth]{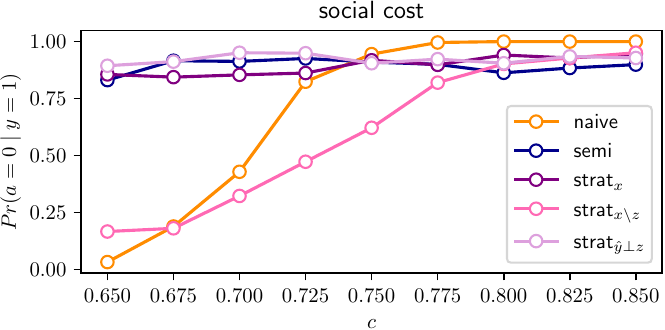}
\caption{\textbf{Social cost} per cost $c$ for different methods. 
}
\label{fig:social_cost}
\end{figure}

\paragraph{Detailed analysis.}
Fig.~\ref{fig:real_plots} shows accuracy and precision curves for varying $\costu$ on a typical \bank\ instance.
Results show how \strat\ is able to consistently perform well by ensuring that only two quality groups apply:
$\grp_1$ (orange), which provides high accuracy,
and $\grp_9$ (cyan), which has lower accuracy but maintains its high precision.
Since $\grp_1$ includes significantly more data points than $\grp_9$, it becomes the dominant component of the overall induced accuracy (black line).
Note how \strat\ increases $\prc_1,\prc_9$ to be above the application threshold as $\costu$ increases;
all other $\prc_\grp$ are substantially lower
(see scatterplot).
In contrast, since \naivemthd\ does not account for $\costu$, the order in which groups apply remains fixed to the order induced by $\prc_\grp$;
as $\costu$ increases, less groups apply, and accuracy varies arbitrarily.
Meanwhile, \stratnoz\
is able to push $\prc_9$ higher, but struggles for the more important $\prc_1$---succeeding for $\costu \le 0.725$, but failing above.
In general, the unavailability of $\grp$ as a feature greatly limits its capacity to control individual $\prc_\grp$.

Fig.~\ref{fig:social_cost} shows \emph{social costs}, defined as the ratio of qualified applicants ($y=1$) that did not apply ($a=0$) over all groups, for increasing costs $c$ and average over splits. As expected, for \stratx\ the social cost is high since the number of applications is very low. And whereas \semi\ and \stratindep\ display similar trends, it is \stratnoz\ that shows favorable costs for low $c$. These are even lower that \naive\, where costs are obtained simply by thresholding on the non-adjusting $\prc_\grp$ values, and are therefore zero for the lowest $c=0.65$.
Together, results suggest that accuracy for the firm comes at a social cost,
which is not mitigated by statistical parity.
This calls for other means for ensuring that qualified candidates apply across all groups---which should benefit both the firm and the candidate population.
\squeeze


\extended{increasing fairness penalty causes $p(\yhat=1)$ to *decrease*, becoming very small}

\section{Discussion} \label{sec:discussion}
This work studies classification under strategic self-selection,
a setting in which  humans---as the subjects of prediction---choose if to participate or not, this in response to the learned classifier used for screening.
Self-selection is a well-studied, well-documented phenomena that is highly prevalent across many social domains;
given its significant role in determining the eventual composition of the applying sub-population in job hiring, school admissions, welfare programs, and other domains,
we believe it is important to understand how self-selection affects, and is affected, by deployed predictive models.
Our work focuses on classifiers used in a particular screening setting,
but learning can influence self-selection in broader manners as well.
One avenue for future work is to enable more fine-grained individualized decisions, e.g., based on private information or beliefs.
Another path is to consider self-selective dynamics
(e.g., in a performative prediction setting).
Yet another path is to determine how to best choose which statistics to reveal.
We leave these for future investigation.


\begin{figure}[t!]
\centering
\includegraphics[width=\columnwidth]{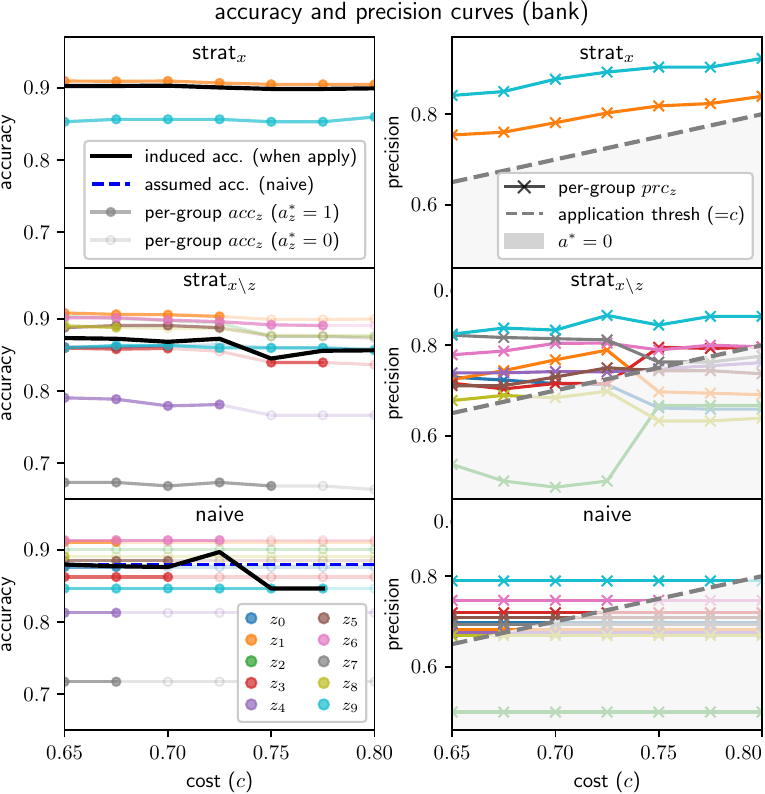}
\includegraphics[width=\columnwidth]{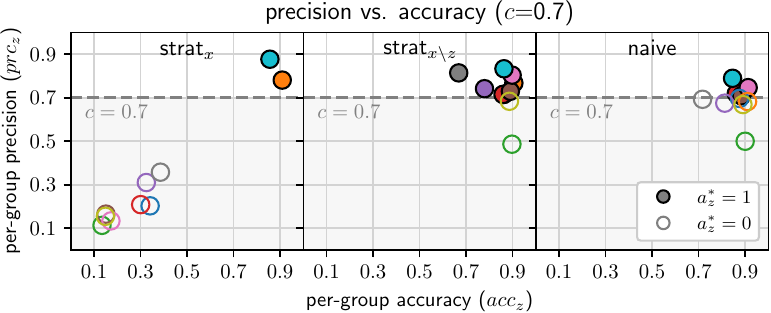}
\caption{
\textbf{(Top:)}
Accuracy and precision curves on an instance of \bank.
\textbf{(Bottom):}
Precision vs. accuracy per $\grp$.
}
\label{fig:real_plots}
\end{figure}
 


\section*{Acknowledgements}
This work is supported by the Israel Science Foundation grant no. 278/22.
\section*{Impact Statement}
A basic assumption in conventional machine learning is that data is drawn from a fixed underlying distribution.
Our work challenges this assumption by positing that in social settings,
the distribution in practice will be shaped by which users choose to participate, this in response to the learned classifier.
We believe this applies widely in settings
where learning is used to support downstream decisions regarding humans,
such as in hiring, admissions, or loan approval.
Our work suggests that overlooking how the learned classifier influences participation can result in unexpected or even undesired outcomes,
and propose means for learning in a way that anticipates and accounts for strategic self-selection.
Our hope is that this will enable firms, governments, and public institutions to learn in ways which balance their goals with those of their potential users.

One implication of our work relates to algorithmic bias.
Amounting evidence suggests that learning systems have the capacity, and often the tendency, to perpetuate social biases that exist in the data.
This is often seen as an undesired artifact of blindly pursuing learning objectives (e.g., maximizing accuracy) without any social considerations.
And whereas much effort has been devoted to developing methods that reduce such data-driven biases (e.g., by enforcing fairness constraints),
our work's perspective provides a possible explanation for \emph{how such biases came to be},
i.e., why the data we observe is biased to begin with.
For example, note it is possible to construct a classifier that is entirely `fair' (under any reasonable notion) when measured on the \emph{induced} test distribution; but where this is so only because the classifier prevented certain groups from applying---groups which otherwise would have made it difficult (or impossible) to achieve fairness.
This alludes to a more subtle form of `implicit' inequity in which classifiers exploit the fact that users make decisions under uncertainty
to create an illusion of fairness and thus circumvent accountability.

A second implication of our work considers its scope.
Note that many of our conclusions relate to the possible discrepancies between
learning in a way that is aware of self-selective behavior, compared to ways that are not.
It is however important to emphasize that these conclusions hold under the simplified setting we consider, and in particular, under rational strategic behavior.
Thus, while we believe our work carries practical implications for learning and policy-making, these must be considered with much care and deliberation as to the feasibility of our assumptions in reality. We are also hopeful that future work will extend beyond our focal setting, whether by considering more general behavioral models, other learning settings, or broader economic environments.
As for the latter, note our work applies mostly to markets where there is scarcity in supply (e.g., jobs) and a relative abundance of demand (e.g., qualified workers).
This justifies why firms in our setting need not worry about `missing' potential candidates---only about hiring good ones. And while many markets are such (e.g., academic positions, executive managers, academic publishing, and in some cases loans), clearly there are markets (or times) in which demand is scarce and supply is abundant (e.g., many jobs but few qualified candidates).
This distinction has implications not only in utilitarian terms, but also in terms of welfare and the role of learning in determining such outcomes.

A final implication regards transparency and the need for responsible usage of learning in social settings.
Our results suggests that, under the setting we consider, learning has the capacity to determine the eventual user population.
This grants learning much power that should not be taken lightly.
As we show, when learning simply pursues the conventional goal of optimizing accuracy, this can come at the expense of preventing most groups from applying (indirectly, by reducing the cost-effectiveness of application).
Since such undesirable outcomes arise even unintentionally---simply as a result of deploying a classifier---any firm that seeks to employ learning for improving decision-making should do so responsibly and transparently.
The latter is crucial given the key role information plays in our setting,
in the forms of what statistics are made public by the firm and therefore shared with potential candidates (e.g., conditional precision metrics).
Note that what information is revealed is a \emph{choice}---one that has major implications on outcomes.
Whereas our work considers this choice as a given, in practice, we believe that it offers a valuable entry point for a social planner or regulator to implement social policy---although how to do this precisely and effectively requires further research efforts.
\squeeze    

Overall, we hope our work aids in raising awareness to self-selection in social settings, its likely prevalence, and the potential capacity of learning to affect it.

\bibliography{refs}
\bibliographystyle{icml2024}

\newpage
\appendix
\onecolumn
\section{Proofs} \label{appx:proofs}

\subsection*{Proposition \ref{prop:utility_simple_form}:}
\begin{proof}
If candidates base their beliefs $\ptilde(y,\yhat \mid x)$ 
on valid information, which in this case amounts to the aggregate per-group statistics published by the firm, namely $P_\smplst(y,\yhat\,|\,\grp)$,%
\footnote{In particular, we assume that candidates do not hold any private information, other than $x$ itself, that is further informative of $p(y,\yhat\,|\,\grp)$.}
then we can rewrite the decision rule as:
\begin{equation}\label{eq:optimal_a}
\begin{aligned}
a^* &= \argmax\nolimits_{a \in \{0,1\}} \expect{\ptilde(y,\yhat \mid x)}{u(a)} \\
&= \argmax\nolimits_{a \in \{0,1\}} \expect{p(y,\yhat \mid \grp)}{u(a)} \\
&= \one{ \expect{p(y,\yhat \mid \grp)}{u(1)} \geq 0} \\
&= \one{ \expect{p(y,\yhat \mid \grp)}{\yhat\cdot(y-c)} \geq 0}
    \end{aligned}
    \end{equation}
Computing the expectation directly, we get:
\begin{align} \label{eq:util_derivation}
\expect{p(y,\yhat\mid \grp)}{\yhat\cdot(y-c)} &=
\sum_{y,\yhat\in \{0,1\}}{p(y,\yhat\mid \grp)\cdot\yhat\cdot(y-c)} \nonumber \\
&= \sum_{y\in \{0,1\}}{p(y,\yhat=0\mid \grp)\cdot0\cdot(y-c) + p(y,\yhat=1\mid \grp)\cdot1\cdot(y-c)} \nonumber \\
&= \sum_{y\in \{0,1\}}{p(y,\yhat=1\mid \grp)\cdot(y-c)} \nonumber \\
&= p(y=0,\yhat=1\mid \grp)\cdot(0-c) + p(y=1,\yhat=1\mid \grp)\cdot(1-c) \nonumber \\
&= p(y=1,\yhat=1\mid \grp) -c\cdot \left(p(y=0,\yhat=1\mid \grp) + p(y=1,\yhat=1\mid \grp)\right) \nonumber \\
&= p(y=1,\yhat=1\mid \grp) -c\cdot p(\yhat=1\mid \grp) \nonumber \\
&= p(\yhat=1\mid \grp)\cdotp(y=1|\yhat=1,\grp) -c\cdot p(\yhat=1\mid \grp) \nonumber \\
&= p(\yhat=1\mid \grp)\left(p(y=1|\yhat=1,\grp) -c\right)  \nonumber \\
&= p(\yhat=1\mid \grp)\left(\prc_\grp -c\right)
\end{align}
Plugging this back into Eq. \eqref{eq:optimal_a}, we get:
\begin{align}\label{eq:applicaiton_function}
a^* &= \one{ p(\yhat=1\mid \grp)\left(\prc_\grp -c\right) \geq 0} \nonumber \\
&= \begin{cases}
        \one{ \prc_\grp \geq c} & \text{ if } p(\yhat=1\mid \grp)>0 \\
        0 & \text{ o.w. } 
\end{cases}
\end{align}
which under the assumption that $p(\yhat=1\mid \grp)>0$ always
simplifies to $a^*=\one{ \prc_\grp \geq c}$.
\end{proof}

\subsection*{Proposition \ref{prop:util_mono_in_prc}:}
\begin{proof}
Immediate from Eq.~\eqref{eq:util_derivation}
\end{proof}

\subsection*{Proposition \ref{prop:util_not_mono_in_prc}:}
\begin{proof}
Proof by construction. Let $\costu=0.5$, fix $m=15$, and consider the following contingency tables:

\begin{equation*}
\begin{tabular}{c|cc}
 $f_1$     & $y=0$ & $y=1$ \\ \hline
$\yhat=0$ &   2        &   3        \\
$\yhat=1$ &   3        &   7       
\end{tabular}
\qquad \qquad \qquad \qquad
\begin{tabular}{c|cc}
 $f_2$     & $y=0$ & $y=1$ \\ \hline
$\yhat=0$ &   0        &   0        \\
$\yhat=1$ &   5        &   10       
\end{tabular}
\end{equation*}

Precision values are $7/(3+7)=0.7$ for $f_1$ and $10/(5+10)=2/3<0.7$ for $f_2$,
so higher for $f_1$.
Recall $u(a)=a\yhat(y-\costu)$, then utilities are:
\begin{align*}
 u_1(1) = - \frac{2}{15}\cdot0\cdot\frac{1}{2} - \frac{3}{15}\cdot1\cdot\frac{1}{2} + \frac{3}{15}\cdot0\cdot\frac{1}{2} + \frac{7}{15}\cdot1\cdot\frac{1}{2} = 0.133 &&& \\
 u_2(1) = - \frac{0}{15}\cdot0\cdot\frac{1}{2} - \frac{5}{15}\cdot1\cdot\frac{1}{2} + \frac{0}{15}\cdot0\cdot\frac{1}{2} + \frac{10}{15}\cdot1\cdot\frac{1}{2} = 0.166 &&& 
\end{align*}
and so higher for $f_2$.
\end{proof}

\begin{figure}[t!]
\centering
\includegraphics[width=0.55\columnwidth]{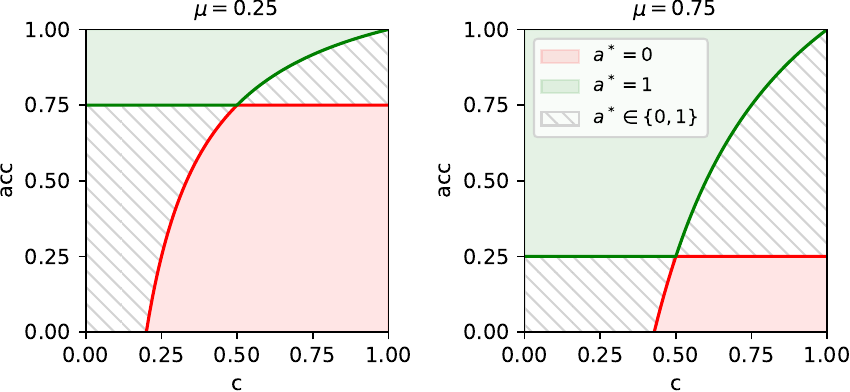}
\caption{
\textbf{Regions of application}.
Plots show for different base rates $\base$ the relations between a classifier's accuracy $\acc_\grp$ on group $\grp$ and that group's application $a^*$. In red regions $a^*=0$, in green regions $a^*=1$, and elsewhere $a^*$ is unconstrained.
}
\label{fig:acc_of_c}
\end{figure}

\subsection*{Proposition~\ref{prop:acc_thresh_apply-or-not} and Corollary~\ref{cor:unconstrained_apply}:}
\begin{proof}
Our proof relies on the main Theorem from \citet{alvarez2002exact} that makes precise the relation between base rate, accuracy, precision, and recall (denoted $\rcl$) of any classifier.
The result states that:
\begin{equation}
\label{eq:alvarez_acc+prc+rcl+base}
\base \cdot \rcl + (\base - \err)\prc = 2 \cdot \base \cdot \rcl \cdot \prc
\end{equation}
In what follows, for clarity we will omit the subscript $\grp$.
Extracting $\prc$ gives:
\[
\prc = \frac{\base \cdot \rcl}{\base(2\rcl-1)+\err}
\]
For the first statement, bounding the RHS from above by $\costu$ will guarantee $\prc \le \costu$ and therefore $a^*=0$. Rearranging gives:
\[
\acc \le 1- \base(\rcl(\frac{1}{\costu}-2)+1)
\]
The RHS is linear in $\rcl$, is increasing if $\costu>1/2$, and decreasing if $\costu<1/2$ (and constant otherwise).
It can therefore be bounded from above by plugging in $\rcl=0$ and $\rcl=1$, and taking the minimum. This gives:
\[
\acc \le \min\{1-\base,1-\base\frac{1-c}{c}\} =
1-\base\max\{1,\base(\frac{1}{c}-1)\}
\,\,\,\Rightarrow\,\,\,
a^*=0
\]
For the second statement, we can bound the RHS in \ref{eq:alvarez_acc+prc+rcl+base} from below by $\costu$;
this will guarantee $\prc < \costu$, and therefore $a^*=1$. 
Similarly, we can take the minimum from $\rcl=0$ and $\rcl=1$, which gives:
\begin{equation}
\label{eq:acc_lb}
\acc > \max\{1-\base,1-\base\frac{1-c}{c}\} =
1-\base\min\{1,\base(\frac{1}{c}-1)\}
\,\,\,\Rightarrow\,\,\,
a^*=1
\end{equation}
For the corollary, note that the above effectively makes use of $\rcl$ as the only degree of freedom;
thus, for any $\acc$ satisfying:
\[
1-\base\max\{1,\base(\frac{1}{c}-1)\} 
< \acc \le 
1-\base\min\{1,\base(\frac{1}{c}-1)\}
\]
then as long as $F$ is sufficiently expressive, $a^*$ remains unconstrained.
\end{proof}



\subsection*{Lemma \ref{lem:calibrated_score_function}:}
\begin{proof}
Let $f_{\score,\thresh}(x)=\one{\score(x)>\thresh}$ where $\score: \X\rightarrow[\alpha,\beta]$. Denote the random variable $\varphi=\score(x)$ and assume that $p(\varphi,y)$ is a well defined density function, so that the marginal and conditional density functions $p_{\varphi}$ and $p_{\varphi\mid y}$ are also well-defined, and further assume that $\supp(p_\varphi)=[\alpha,\beta]$ (this to ensure that $\varphi$ also induces a well-defined density).
Denote by $\prc(\thresh)$ the precision of $f_{\score,\thresh}(x)$ as a function of the threshold $\thresh \in [\alpha,\beta)$, i.e.:
$$\prc(\thresh) = P(y=1|\score(x)>\thresh)=P(y=1|\varphi>\thresh)$$
Note that $\prc(\thresh)$ is well defined for all $\thresh<\beta$, because for these values of $\thresh$, $P(\varphi>\thresh)>0$ since $\supp(p_\varphi)=[\alpha,\beta]$.
With Bayes theorem we can write:
\begin{equation*}
    \prc(\thresh) = P(y=1)\cdot\frac{P(\varphi>\thresh|y=1)}{P(\varphi>\thresh)}
\end{equation*}
If $P(y=1)=0$, then $\prc(\thresh)=0$ for any $\thresh$, therefore $\prc(\thresh)$ is (weakly) monotonically increasing.

Now, assume that $P(y=1)>0$.
Since $p(\varphi,y)$ is a well-defined density function, $P(\varphi>\thresh|y=1)$ and $P(\varphi>\thresh)$ are differentiable, therefore
$\prc(\thresh)$ is differentiable.
Therefore, $\prc(\thresh)$ is monotonically increasing $\Leftrightarrow \prc'(\thresh)\geq 0$.
Since $P(y=1)$ is a constant, the derivative is:
\begin{equation*}
    \prc'(\thresh) = P(y=1)\cdot\left(\frac{P(\varphi>\thresh|y=1)}{P(\varphi>\thresh)}\right)'
\end{equation*}
Since $P(y=1)>0$, $\prc'(\thresh)\geq 0 \Leftrightarrow \left(\frac{P(\varphi>\thresh|y=1)}{P(\varphi>\thresh)}\right)'\geq0$.
Using the quotient rule, we get that
\begin{equation*}
\left(\frac{P(\varphi>\thresh|y=1)}{P(\varphi>\thresh)}\right)' = \frac{(P(\varphi>\thresh|y=1))'\cdot P(\varphi>\thresh) - P(\varphi>\thresh|y=1) \cdot (P(\varphi>\thresh))'}{(P(\varphi>\thresh))^2}
\end{equation*}
Since $P(\varphi>\thresh)>0$, the denominator is positive, therefore
\begin{equation}\label{eq:derivative}
\left(\frac{P(\varphi>\thresh|y=1)}{P(\varphi>\thresh)}\right)'\geq 0 \Leftrightarrow (P(\varphi>\thresh|y=1))'\cdot P(\varphi>\thresh) - P(\varphi>\thresh|y=1) \cdot (P(\varphi>\thresh))' \geq0
\end{equation}
Using the connection between CDF and PDF, we get
\begin{equation*}
    (P(\varphi>\thresh))' = (1-P(\varphi\leq \thresh))' = -p_{\varphi}(\thresh)
\end{equation*}
and
\begin{equation*}
    (P(\varphi>\thresh|y=1))' = (1-P(\varphi\leq \thresh|y=1))' = -p_{\varphi|y=1}(\thresh)
\end{equation*}
Plugging these in Eq. \eqref{eq:derivative}:
\begin{align*}
    &\prc'(\thresh) \geq0
   &\Leftrightarrow -p_{\varphi|y=1}(\thresh)\cdot P(\varphi>\thresh) + p_{\varphi}(\thresh)\cdot P(\varphi>\thresh|y=1)\geq0
\end{align*}
With Bayes theorem applied to $P(\varphi>\thresh|y=1)$, we get that this holds iff:
\begin{align*}
   & -p_{\varphi|y=1}(\thresh)\cdot P(\varphi>\thresh) + p_{\varphi}(\thresh)\cdot P(\varphi>\thresh)\cdot \frac{P(y=1|\varphi>\thresh)}{P(y=1)}\geq0 \\
   &\Leftrightarrow P(\varphi>\thresh) \cdot \left(-p_{\varphi|y=1}(\thresh) + p_{\varphi}(\thresh)\cdot \frac{P(y=1|\varphi>\thresh)}{P(y=1)}\right)\geq0
\end{align*}
Since $P(\varphi>\thresh)>0$, this holds iff:
\begin{align*}
   &-p_{\varphi|y=1}(\thresh) + p_{\varphi}(\thresh)\cdot \frac{P(y=1|\varphi>\thresh)}{P(y=1)} \geq 0
\end{align*}
With Bayes theorem applied to $p_{\varphi|y=1}(\thresh)$, we get that this holds iff:
\begin{align*}
   & -p_{\varphi}(\thresh)\cdot\frac{P(y=1|\varphi=\thresh)}{P(y=1)} + p_{\varphi}(\thresh)\cdot \frac{P(y=1|\varphi>\thresh)}{P(y=1)} \geq 0 \\
   & \Leftrightarrow \frac{p_{\varphi}(\thresh)}{P(y=1)} \cdot \Big( -P(y=1|\varphi=\thresh) + P(y=1|\varphi>\thresh) \Big) \geq 0
\end{align*}
and since $p_{\varphi}(\thresh)>0$ and $P(y=1)>0$, we further get that this holds iff:
\begin{align*}
    & -Pr(y=1|\varphi=\thresh) + P(y=1|\varphi>\thresh) \geq 0 \\
    & \Leftrightarrow 
    P(y=1|\varphi>\thresh) \geq P(y=1|\varphi=\thresh) \\ 
    &\Leftrightarrow P(y=1|\score(x)>\thresh) \geq P(y=1|\score(x)=\thresh) 
\end{align*}
Overall, we get that $\prc'(\thresh)\geq 0 \Leftrightarrow P(y=1|\score(x)>\thresh) \geq P(y=1|\score(x)=\thresh)$, i.e. $\prc$ is monotonically increasing if and only if $\score$ is a calibrated score function.

\end{proof}

\subsection*{Corollary \ref{cor:calib=>a_is_step_func}:}
\begin{proof}
Let $\grp$ be some group, and let $\score$ be a score function with range $[\alpha,\beta]$ that is calibrated w.r.t. $\grp$. Denote $a^*_\grp(\thresh)=\one{\prc_{\grp}(\thresh)\geq c}$ as the optimal application function of $\grp$ w.r.t. the classifier $f_{\score,\thresh}(x)$.
From Lemma \ref{lem:calibrated_score_function} we get that $\prc_{\grp}(\thresh)$ is monotonically increasing in $\thresh$.
Assume that $a^*_\grp(\thresh)$ is neither a constant or a step function.
Therefore, there exists $\alpha\leq\thresh_1<\thresh_2<\thresh_3<\beta$, such that $a^*_\grp(\thresh_1)\neq a^*_\grp(\thresh_2)$ and $a^*_\grp(\thresh_2)\neq a^*_\grp(\thresh_3)$.
\begin{itemize}
    \item If $a^*_\grp(\thresh_2)=1$, then $a^*_\grp(\thresh_3)=0$.
    Therefore, $\prc_{\grp}(\thresh_2)\geq c$ and $\prc_{\grp}(\thresh_3)< c$, hence $\prc_{\grp}(\thresh_2) > \prc_{\grp}(\thresh_3)$, and since $\thresh_3 > \thresh_2$ this is a contradiction of the monotonicity of $\prc_{\grp}(\thresh)$.
    \item If $a^*_\grp(\thresh_2)=0$, then $a^*_\grp(\thresh_1)=1$.
    Therefore, $\prc_{\grp}(\thresh_1)\geq c$ and $\prc_{\grp}(\thresh_2)< c$, hence $\prc_{\grp}(\thresh_1) > \prc_{\grp}(\thresh_2)$, and since $\thresh_2 > \thresh_1$ this is a contradiction of the monotonicity of $\prc_{\grp}(\thresh)$.
\end{itemize}
Therefore, $a^*_\grp(\thresh)$ is either a constant or a step function. 
\end{proof}

\subsection*{Corollary \ref{cor:score_ordering}:}
\begin{proof}
Let $\score$ be a mutually calibrated score function with range $[\alpha,\beta]$, and let $\grp,\grp'$ such that $\grp \preceq_\score \grp'$.
From the definition of $\preceq_\score$, this implies that $\costu_{\grp}\leq\costu_{\grp'}$.
Let $\thresh\in [\alpha, \beta]$ be such that candidates from $\grp'$ apply, i.e. $a^*_{\grp'}(\thresh)=1$.
Therefore, $\prc_{\grp'}(\thresh)\geq c$, and from the definition of $\costu_{\grp'}$ we get $\thresh\geq\costu_{\grp'}\geq\costu_{\grp}$.
Since $\score$ is mutually calibrated, it is calibrated w.r.t. $\grp$, and from Lemma \ref{lem:calibrated_score_function}, $\prc_{\grp}(\thresh)$ is monotonically increasing, therefore $\prc_{\grp}(\thresh)\geq\prc_{\grp}(\costu_{\grp})$.
From the definition of $\costu_{\grp}$, $\prc_{\grp}(\costu_{\grp})\geq \costu$, therefore $\prc_{\grp}(\thresh)\geq\costu$.
Therefore $a^*_{\grp}(\thresh)=1$, i.e. candidates from $\grp$ apply.
\end{proof}

\subsection*{Proposition \ref{prop:strategic_single_group_applies}:}
\begin{proof}
Consider first the case where all $\grp$ are such that $\base_\grp < \costu$. As we have shown, if for some group $\grp$ we have $\yhat=1$ for all $x$ in that group, then $\prc_\grp=\base_\grp$,
and as a result, $a^*_\grp=0$.
To obtain the optimal classifier:
(i) go over all groups, and for each group $\grp$,
train $f^\grp$ on the subset of data from $\grp$ alone;
(ii) find $\grp^*$ such that $f^{\grp^*}$ has the highest accuracy;
and (iii) construct the final classifier $f(x)=g(x)+v^\top \grp$
as $g=f^{\grp^*}$, $v_{\grp^*}=0$, and $v_\grp=-\infty$ for all $\grp \neq \grp^*$. This ensures that only $\grp^*$ applies, and $f$ obtains the same accuracy (on this group, and hence on all applicants) as $f^{\grp^*}$. Note that any other $f'$ for which other groups apply can have accuracy at most that of $f^{\grp^*}$, 
since this would include averaging over additional groups.
Hence, the constructed $f$ is optimal.
In the more general case where some groups may have $\base_\grp>\costu$, it is impossible to guarantee that only a specific group applies (in particular when monotonicity does not hold); however, for all groups that do have $\base_\grp < \costu$, the same reasoning as before still applies.
Note that in this case the proof is not constructive, since the task of inferring the optimal subset of groups becomes combinatorially challenging.
\end{proof}

\subsection*{Corollary \ref{cor:specific_groups}:}
\begin{proof}
Let $f$.
As before, if for some group $\grp$  with $\base_\grp<c$ we set $v_\grp=-\infty$, then this group will not apply. Furthermore, the behavior of $f$ on any $x$ from other $\grp' \neq \grp$ will remain the same.
\end{proof}

\subsection*{Proposition \ref{prop:indep_order}:}
\begin{proof}
Let $f$, and assume $f(x) \perp \grp$, therefore $\yhat \perp \grp$.
Let $\grp,\grp'$ such that $\grp \preceq_\base \grp'$, i.e. $\base_\grp \ge \base_{\grp'}$.
Assume that $\grp'$ applies under $f$, i.e. $\prc_{\grp'}\geq c$.
For clarity we will write $P$ to mean $P_\smplst$.
With Bayes theorem can express $\prc_{\grp}$ as:
\begin{equation*}
\begin{aligned}
    \prc_{\grp} &= P(y=1 \mid \yhat=1, \grp) = P(\yhat=1 \mid y=1, \grp) \frac{P(y=1 \mid \grp)}{P(\yhat=1\mid \grp)} \\
    &= P(\yhat=1 \mid y=1, \grp) \frac{\mu_{\grp}}{P(\yhat=1\mid \grp)}
\end{aligned}
\end{equation*}
Since $\yhat \perp \grp$, $P(\yhat=1\mid \grp)= P(\yhat=1)$ and
$P(\yhat=1 \mid y=1, \grp) = P(\yhat=1 \mid y=1)$,
so we get:
\begin{equation*}
\begin{aligned}
    \prc_{\grp} &= P(\yhat=1 \mid y=1) \frac{\mu_{\grp}}{P(\yhat=1)}
\end{aligned}
\end{equation*}
Using Bayes theorem again on $P(\yhat=1 \mid y=1)$, we get:
\begin{equation*}
\begin{aligned}
    \prc_{\grp} &= P(y=1 \mid \yhat=1) \frac{P(\yhat=1)\cdot\mu_{\grp}}{P(y=1)\cdot P(\yhat=1)} \\
    &= P(y=1 \mid \yhat=1) \frac{\mu_{\grp}}{P(y=1)} \\
    &= \prc \frac{\mu_{\grp}}{\mu}
\end{aligned}
\end{equation*}
where $\prc$ and $\mu$ are the global precision and base rate, respectively.
In the same way, $\prc_{\grp'}=\prc \frac{\mu_{\grp'}}{\mu}$.
Therefore, 
\begin{equation*}
    \prc_{\grp} = \prc \frac{\mu_{\grp}}{\mu} \geq \prc \frac{\mu_{\grp'}}{\mu} = \prc_{\grp'}
\end{equation*}
And since $\prc_{\grp'}\geq c$, we get that $\prc_{\grp}\geq c$, i.e. $\grp$ applies.
\end{proof}

\section{Additional results and illustrations.}

\subsection{Calibrated score functions vs. calibrated classifiers} \label{appx:calibration}
Our next result connects score function calibration (Definition~\ref{def:calibration})
to the standard notion of calibration for probabilistic classifiers.
In particular, we show that under the assumption that $p(\score(x),y)$ is a well-defined density, score function calibration is a weaker requirement,
and is therefore implied by standard calibration.

\begin{lemma} \label{lem:if_calibrated_then_calibrated_score}
Let $g(x)=\phat(y=1|x)$ be a calibrated probabilistic classifier,
i.e., satisfies $P(y=1 \mid \phat=\thresh) = \thresh$.
Then $\score=g$ is a calibrated score function (with range $[0,1]$).
\end{lemma}

\begin{proof}
Let $\varphi=\score(x)$ be a random variable depicting scores, and assume that $p(\varphi,y)$ is a well defined density function, so that the marginal density function $p_{\varphi}$ is also well-defined.
If we think of $\score$ as a probabilistic classifier $h$, then
from the definition of classifier calibration we get $P(y=1|\varphi=\thresh)=\thresh$, where $\thresh\in[0,1]$.
Therefore, by the definition of conditional probability we get:
\begin{equation}\label{eq:calibration_def_simple}
    p_{y,\varphi}(y=1,\varphi=\thresh)=\thresh\cdot p_\varphi(\thresh)
\end{equation}
Using the definition of conditional probability again, we get:
\begin{equation*}
    P(y=1|\varphi>\thresh) = \frac{Pr(y=1,\varphi>\thresh)}{P(\varphi>\thresh)}
\end{equation*}
Using the law of total probability, we get:
\begin{equation*}
    \frac{P(y=1,\varphi>\thresh)}{P(\varphi>\thresh)} = \frac{\int_{t=\thresh}^{1}p_{y,\varphi}(y=1,\varphi=t)\d t}{P(\varphi>\thresh)}
\end{equation*}
Plugging in Eq. \eqref{eq:calibration_def_simple}, we get:
\begin{equation*}
    \frac{\int_{t=\thresh}^{1}p_{y,\varphi}(y=1,\varphi=t)\d t}{P(\varphi>\thresh)} = 
    \frac{\int_{t=\thresh}^{1}t\cdot p_\varphi(t)\d t}{P(\varphi>\thresh)} \geq 
    \frac{\int_{t=\thresh}^{1}\thresh\cdot p_\varphi(t)\d t}{P(\varphi>\thresh)} = 
    \frac{\thresh\cdot P(\varphi>\thresh)}{P(\varphi>\thresh)} = \thresh
\end{equation*}
Finally, this gives:
\begin{equation*}
    P(y=1|\score(x)>\thresh) \geq \thresh = P(y=1|\score(x)=\thresh)
\end{equation*}
which means that $\score$ is a calibrated score function over $[0,1]$.
\end{proof}

\subsection{Calibrated score functions vs. within group monotone classifiers} \label{appx:within_group_monotone}
Our next result connects score function calibration (Definition~\ref{def:calibration})
to the definition of \emph{within-group monotonicity} from \citet{okati2023within}.
Using our notations, their definition can be stated as:
\begin{definition}[Within-group monotonicity]
Let $\score$ be a score function with range $[\alpha,\beta]$. Then $\score$ is within-group monotone if for any $\grp$ and for any $\alpha\leq\thresh<\thresh'<\beta$ such that $p(\grp\mid\score(x)=\thresh)>0$ and $p(\grp\mid\score(x)=\thresh')>0$, it holds that :
\begin{equation*}
    P(y=1\mid \score(x)=\thresh,\,\grp) \leq P(y=1\mid \score(x)=\thresh',\,\grp)
\end{equation*}
\end{definition}
In some sense, this definition is more general than our definition of mutually calibrated score functions, because it does not require a well-defined density $p(\score(x),y)$.
However, we show that within our setting (in which we do assume that $p(\score(x),y)$ is a well-defined density),
our notion of score function calibration turns out to be a \emph{strictly weaker} requirement than within-group monotone, and is implied by it.
For simplicity, and w.l.o.g., we will prove this for a version within-group monotonicity that considers a general distribution $p$, i.e., absent the conditioning on $\grp$:
$P(y=1\mid \score(x)=\thresh) \leq P(y=1\mid \score(x)=\thresh')$.%
\footnote{Given this proof, the proof for Lemma \ref{lem:if_calibrated_then_calibrated_score} can be simplified, since a calibrated probabilistic classifier is a special case of a within-group monotone classifier, thus it implies score function calibration.}

We begin by showing that monotonicity implies calibration,
and then providing an example of a score function that is calibrated but is not monotone.

\begin{lemma} \label{lem:if_within_group_monotone_then_calibrated_score}
Let $\score(x)$ be a score function with range $[\alpha,\beta]$, such that $p(\score(x),y)$ is a well-defined density, and $\score$ has full support on $[\alpha,\beta]$ under $p$. 
Then if $\score$ is within-group monotone,
it is also a calibrated score function.
\end{lemma}

\begin{proof}
Let $\varphi=\score(x)$ be a random variable depicting scores,
and denote its density function by $p_{\varphi}$.
Let $\thresh\in[\alpha,\beta)$.
Using Bayes theorem we can write:
\begin{equation*}
\begin{aligned}
    P(y=1\mid \varphi>\thresh) &= \frac{P(y=1)}{P(\varphi>\thresh)}P(\varphi>\thresh\mid y=1) \\
    &= \frac{P(y=1)}{P(\varphi>\thresh)}\int_{t=\thresh}^{\beta}p_{\varphi\mid y=1}(t)\d t
\end{aligned}
\end{equation*}
Using Bayes theorem again on $p_{\varphi\mid y=1, \grp}(\thresh)$ we get:
\begin{equation*}
\begin{aligned}
    P(y=1\mid \varphi>\thresh) &= \frac{P(y=1)}{P(\varphi>\thresh)}\int_{t=\thresh}^{\beta}P(y=1\mid \varphi=t)\frac{p_{\varphi}(t)}{P(y=1)}\d t \\
    &= \frac{1}{P(\varphi>\thresh)}\int_{t=\thresh}^{\beta}P(y=1\mid \varphi=t)\cdot p_{\varphi}(t)\d t
\end{aligned}
\end{equation*}
Since $\score$ is within-group monotone, for all $t>\thresh$ it holds that $P(y=1\mid \varphi=t) \geq P(y=1\mid \varphi=\thresh)$, therefore:
\begin{equation*}
\begin{aligned}
    P(y=1\mid \varphi>\thresh) &\geq \frac{1}{P(\varphi>\thresh)}\int_{t=\thresh}^{\beta}P(y=1\mid \varphi=\thresh)\cdot p_{\varphi}(t)\d t \\
    &= \frac{P(y=1\mid \varphi=\thresh)}{P(\varphi>\thresh)}\int_{t=\thresh}^{\beta} p_{\varphi}(t)\d t \\
    &= \frac{P(y=1\mid \varphi=\thresh)}{P(\varphi>\thresh)} P(\varphi>\thresh) \\
    &= P(y=1\mid \varphi=\thresh)
\end{aligned}
\end{equation*}
Therefore $\score$ is a calibrated score function.
\end{proof}

\begin{lemma} \label{lem:within_group_monotone_strictly–stronger}
If $p(\score(x),y)$ is well-defined, then within-group monotonicity is strictly stronger than score function calibration.
\end{lemma}

The proof is based on a constructive example of a score function that is calibrated, but is not within-group monotone.
Let $x\in \R$, such that $x\sim {U}(0,1)$, and let $P(y=1, x)=(x-\frac{1}{3})^2+\frac{1}{3}$ for $x\in[0,1]$.
Let $\score(x)=\varphi=x$ be the identity score function.
We get that for $\thresh\in[0, 1)$,
\begin{equation*}
    P(y=1\mid \varphi=\thresh)=\frac{P(y=1, x=\thresh)}{p_x(\thresh)} = \frac{\left(\thresh-\frac{1}{3}\right)^2+\frac{1}{3}}{1}=\left(\thresh-\frac{1}{3}\right)^2+\frac{1}{3}
\end{equation*}
For $\thresh=0$ and $\thresh'=\frac{1}{3}$, we get that $\thresh<\thresh'$, but
\begin{equation*}
    P(y=1\mid \varphi=\thresh)=\frac{1}{3} + \frac{1}{9} > \frac{1}{3} = P(y=1\mid \varphi=\thresh')
\end{equation*}
therefore $\score$ is not within-group monotone.
However, $\score$ is a calibrated score function:
\begin{equation*}
    P(y=1\mid \varphi>\thresh)=\frac{P(y=1, x>\thresh)}{P(x>\thresh)} = \frac{\int_{x=\thresh}^{1}  \left(x-\frac{1}{3}\right)^2+\frac{1}{3}\d x}{1-\thresh} = \frac{-3\thresh^3+3\thresh^2-4\thresh+4}{9(1-\thresh)}
\end{equation*}
As can be seen in Figure \ref{fig:calibrated_score} below, $P(y=1\mid\varphi>\thresh) \geq P(y=1\mid\varphi=\thresh)$ for all $\thresh\in [0,1)$.

\vspace{2em}

\begin{SCfigure}[2][h!]
\centering
\includegraphics[width=0.3\textwidth]{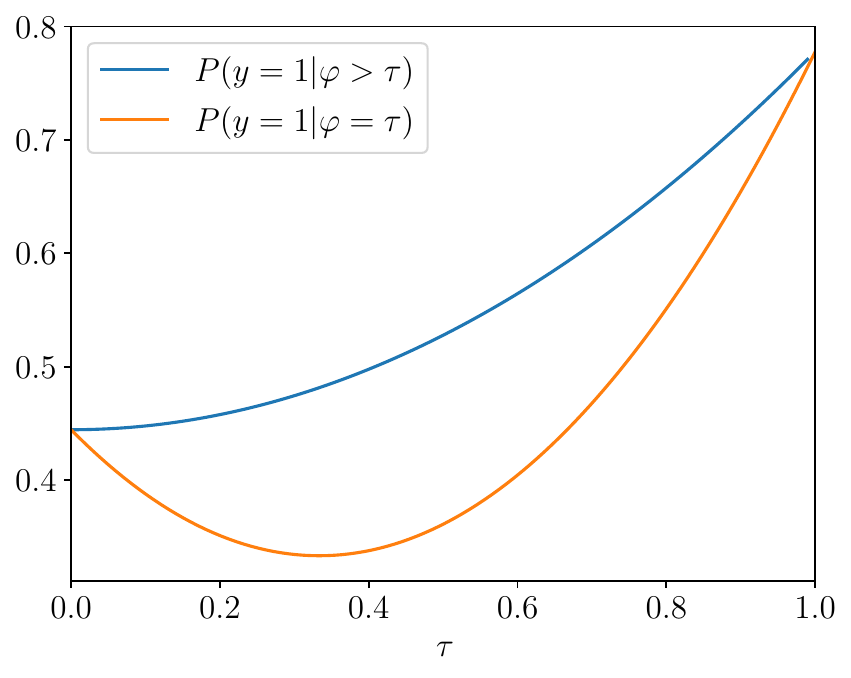}
\caption{
\textbf{A calibrated score function}.
Plot shows data distribution with a score function $\score(x)=\varphi=x$, where $P(y=1\mid\varphi=\thresh) = \left(\thresh-\frac{1}{3}\right)^2+\frac{1}{3}$, and $P(y=1\mid\varphi>\thresh) = \frac{-3\thresh^3+3\thresh^2-4\thresh+4}{9(1-\thresh)}$.
Under this distribution, $\score$ is a calibrated score function: $P(y=1\mid\varphi>\thresh) \geq P(y=1\mid\varphi=\thresh)$ for all $\thresh\in [0,1)$.
However, $\score$ is not within-group monotone: as can be seen, $P(y=1\mid\varphi=\thresh)$ is decreasing for $\thresh\in[0, \frac{1}{3}]$.
}
\label{fig:calibrated_score}
\end{SCfigure}


\begin{figure}[t!]
\centering
\includegraphics[width=0.6\columnwidth]{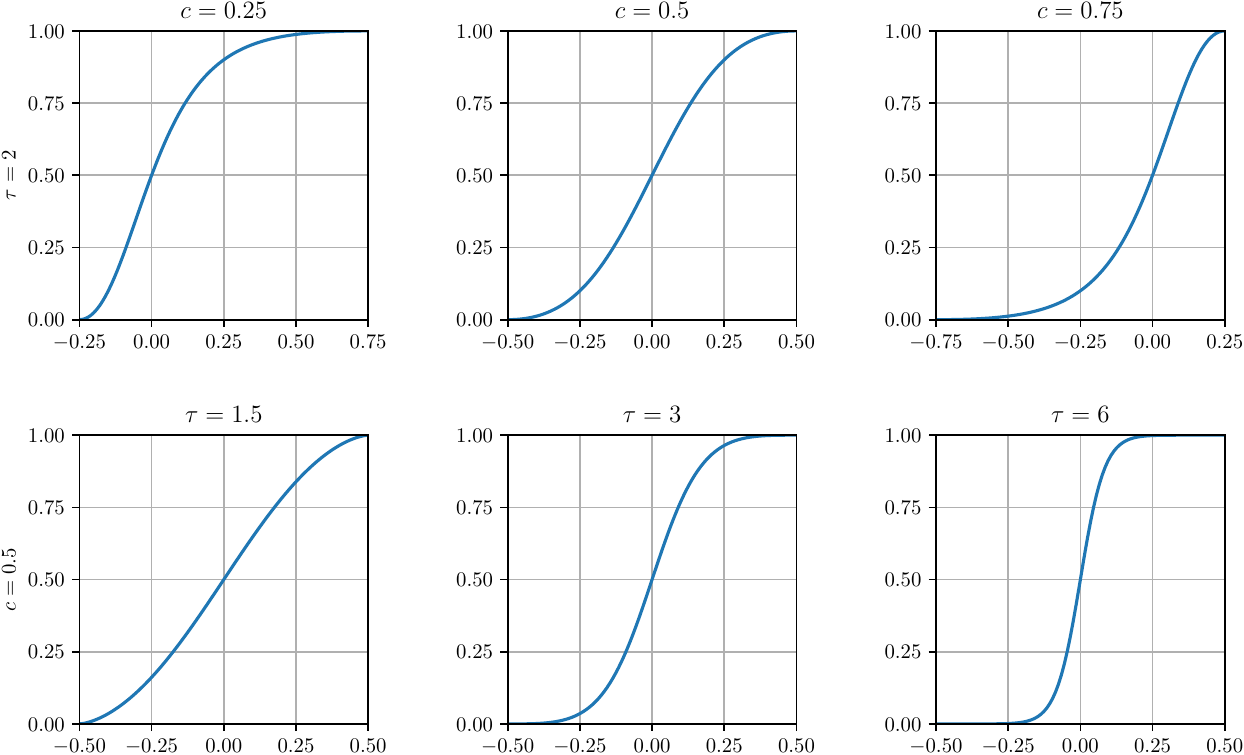}
\caption{
\textbf{Proposed sigmoid function $\sig$ for smoothed applications.}
}
\label{fig:sigmoid}
\end{figure}

\subsection{A sigmoid for smoothed applications} \label{appx:sigmoid}
Fig.~\ref{fig:sigmoid} illustrates our proposed sigmoid $\sig$ in 
Eq.~\eqref{eq:sigmoid} from Sec.~\ref{sec:method}, intended to serve as a differentiable proxy for applications.
The figure shows how the sigmoid's shape is affected by the cost parameter $\costu$ and the temperature parameter $\temp$.
The top row illustrates the role of $\costu$ (for fixed $\temp$).
Note how the domain is $-\costu,1-\costu$, and so shifts with $\costu$, 
but the indifference point (for which the output is 0.5) remains at input 0.
The bottom row shows how, for fixed $\costu$, increasing $\tau$ increases the sigmoid's slope, this enabling $\sig$ to better approximate a step function,
but making it harder to optimize.

\subsection{A corrective term for the smoothed precision proxy} \label{appx:corrective_term}
First, note we can write precision as:
\begin{equation*}
a^*= \one{\prc>c} = \one{\frac{\sum_i y_i \yhat_i}{\sum_i \yhat_i} > c} =
\one{ \sum_i y_i \yhat_i > c\sum_i \yhat_i  }
\end{equation*}
and therefore the smoothed precision proxy as:
\begin{equation*}
\atilde= \one{\softprc >c} = \one{ \sum_i y_i \ytilde_i > c\sum_i \ytilde_i  } \end{equation*}
Using the definition of the corrective term:
\begin{equation*}
\bias = \frac{1}{\costu}\sum_i (y_i-c)(\yhat_i-\ytilde_i)
\end{equation*}
rearranging the expression in the indicator for hard precision gives:
\begin{align*}
\sum_i y_i \yhat_i - c\sum_i \yhat_i =
\sum_i y_i (\yhat_i + \ytilde_i - \ytilde_i) - c\sum_i (\yhat_i + \ytilde_i -\ytilde_i) 
= \sum_i y_i \ytilde_i - c (\sum_i \ytilde_i  + \bias)
\end{align*}
From this, we can derive the corrected soft prediction decision rule:
\begin{align*}
\atilde = \one{\sum_i y_i \ytilde_i > \costu (\sum_i \ytilde_i - \bias) } 
= \one{\frac{\sum_i y_i \ytilde_i}{\sum_i \ytilde_i - \bias} > \costu }
\end{align*}



\section{Synthetic experimental results}

\begin{figure}[t!]
\centering
\includegraphics[width=\columnwidth]{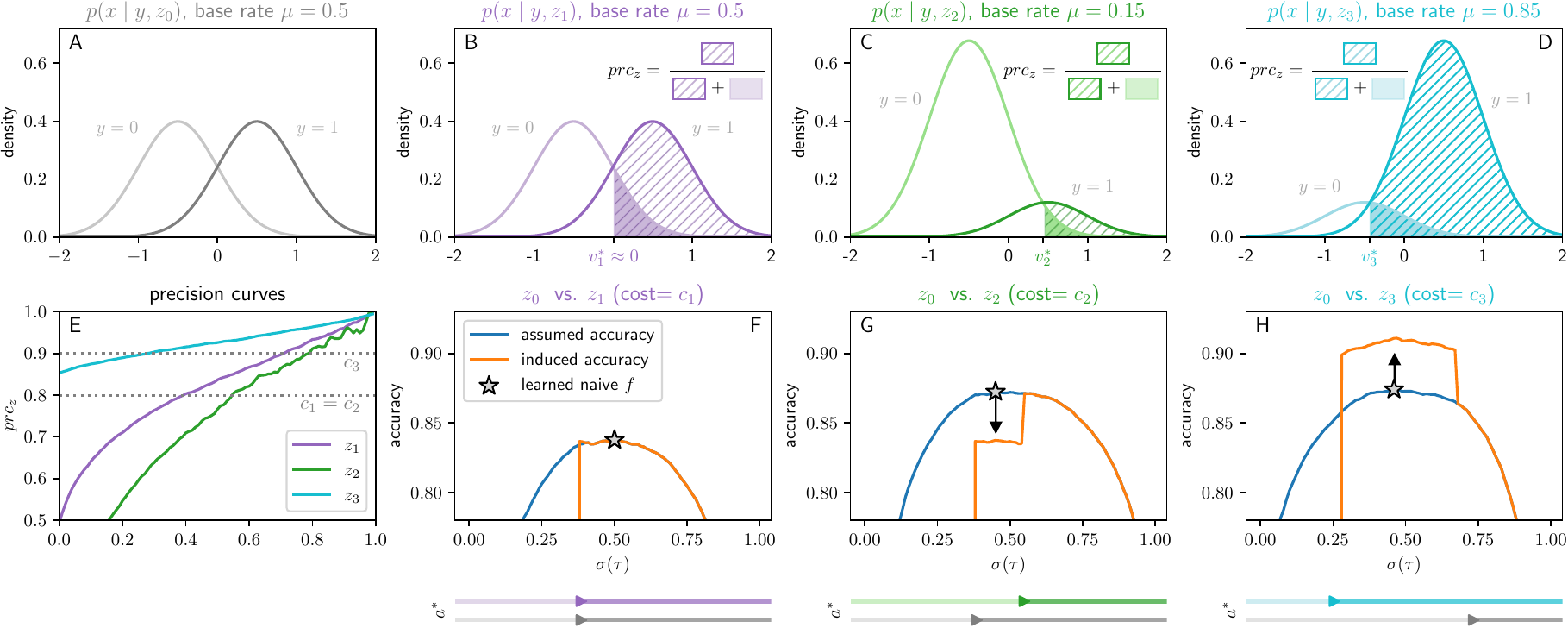}
\caption{
\textbf{Assumed vs. induced accuracy for a naive classifier.}
Each experimental condition includes two groups: $\grp_0$, and one of
$\{\grp_1,\grp_2,\grp_3\}$.
\textbf{(A)-(D):} Per-group class-conditional distributions, varying only by base rate $\base$. Shaded/hatched areas show how precision is computed.
\textbf{(E):} Precision curves for each group.
\textbf{(F)-(G)}: Assumed accuracy vs. actual induced accuracy (on the applicant population) in each experimental condition.
Curves show accuracies for varying thresholds $\thresh$.
Also shown are the actual learned thresholds (star), showing how the assumed accuracy of a \naive\ classifier can be either correct (F), overly optimistic (and wrong) (G), or overly pessimistic (but correct) (H).
For further details see the text in Appendix~\ref{appx:naive}.
\squeeze
}
\label{fig:naive}
\end{figure}

\subsection{Naive learning -- assumed vs. induced accuracy} \label{appx:naive}
The goal of this experiment is to demonstrate how a \naive\ learner that is oblivious to strategic self-selection can end up with a classifier who's actual performance on the induced self-selective distribution is essentially arbitrary:
it can be as expected, lower than expected, or better than expected.

Denoting by $\xbar = x \setminus \grp$ the non-group features,
here we use $\xbar \in \R$.
We consider four groups: $\grp_1,\grp_2,\grp_3,\grp_4$,
and in each instance of the experiment include two groups:
$\grp_0$, and an additional $\grp_i$, where $i \in \{1,2,3\}$.
Each group $\grp$ is associated with distribution
$p(\xbar, y \,|\, \grp)=p(\xbar \,|\,y,\grp) p(y\,|\,\grp)$. 
We define the class-conditional distributions to be Gaussian
with $\xbar \mid y=0 \sim \mathcal{N}(-0.5,\,0.5)$ and
with $\xbar \mid y=1 \sim \mathcal{N}(0.5,\,0.5)$,
which are fixed across groups,
and let $\base_i=p(y=1\,|\,\grp_i)$ be group-specific
(i.e., do not depend on $\grp$),
and where $\base_i$ is the base rate of group $\grp_i$,
and $p(y=0\,|\,\grp_i)=1-\base_i$.
In particular, we set
$\base_1=\base_2=0.5, \base_3=0.15$, and $\base_4=0.85$;
thus, compared to $\grp_1$, $\grp_2$ has the same base rate,
$\grp_3$ has lower base rate, and $\grp_4$ has higher base rate.
For the first two experiments we use $\costu_1=\costu_2=0.8$, and for the last we use $\costu_3=0.9$ (we explain why below).

In each instance we sample $m=10,000$ examples and `{\naive}ly' train a linear model using logistic regression on the full dataset including all groups.
We then measure the learned classifier's \emph{assumed accuracy} (i.e., made under the assumption that there is no self-selection), and its actual \emph{induced accuracy} on the applicant population.
We also vary the decision threshold $\thresh$ and report both assumed accuracy and induced accuracy on the entire (probabilistic) range $\sigalt(\thresh) \in [0,1]$.
Note that because $\xbar$ is uni-dimensional, a linear model
$w^\top x + b$ can be rewritten as $a\cdot\xbar + v^\top \grp + b$
where $a,b \in \R$ and $v \in \R^2$.
Hard predictions $\yhat=f(x)$ (on which precision and accuracy depend) therefore rely only on the per-group offsets $v_i$, where w.l.o.g. we can assume $b=0$.
Thus, learning $f$ amounts to learning how to "shift" each group's conditional distribution by its corresponding $v^i$ so that a global threshold $\thresh=b=0$ performs well.
We denote the learned group-specific offset terms by $v^*_i$.
\squeeze

Results are shown in Figure~\ref{fig:naive}.
Subplots (A)-(D) show the data distributions of the different groups,
overlaid with an illustration of how precision is computed under the learned model, i.e., for each $v^*_i$.
As can be seen, a lower base rate (as for $\grp_2$ in (C))
causes $v^*_2$ to increase (i.e., shift right).
For a fixed base rate, increasing the threshold should \emph{increase} precision;
however, the smaller base rate causes precision to generally \emph{decrease}, and this effect is stronger.
The result is a lower precision curve (compared to $\grp_1$), which is shown in (E).
In contrast, an increased base rate (as for $\grp_3$ in (D)) \emph{increases} precision---this time leading to a higher precision curve (also shown in (E)).

Sobplots (F)-(H) show assumed and induced accuracy for each experiment and for a range of (global) thresholds, and below each plot are shown the points along the threshold axis in which each group applies.
For $z_1$ (F), whose distribution matches that of $z_0$ (since $\base_1=\base_0$),
we see that the assumed accuracy matches induces accuracy---this is since for the learned $f$ both groups apply, and so the \naive\ perspective turned out to be correct.
For $\grp_2$ (G), assumed accuracy is \emph{higher} than the actual induced accuracy,
since at this point only one group applies (here, $\grp_0$),
which is precisely the result of the lowered precision curve (due to the lower $\base_2$).
For $\grp_3$ (H), for which we used a larger cost ($\costu_3=0.9$), 
assumed accuracy is \emph{lower} than the actual induced accuracy.
Here again this is only since one group applies (this time $\grp_3$),
though now due to the higher precision curve,
which both `kicks in' earlier, and provides higher accuracy at that point.
Note for $\grp_3$, though the assumed accuracy differs from the actual induced accuracy, the learned classifier is nonetheless optimal also for the induced distribution.
In contrast, for $\grp_2$, the \naive\ classifier is suboptimal,
since the optimal classifier requires a larger threshold which ensures that both groups apply.


\begin{figure}[t!]
\centering
\includegraphics[width=\columnwidth]{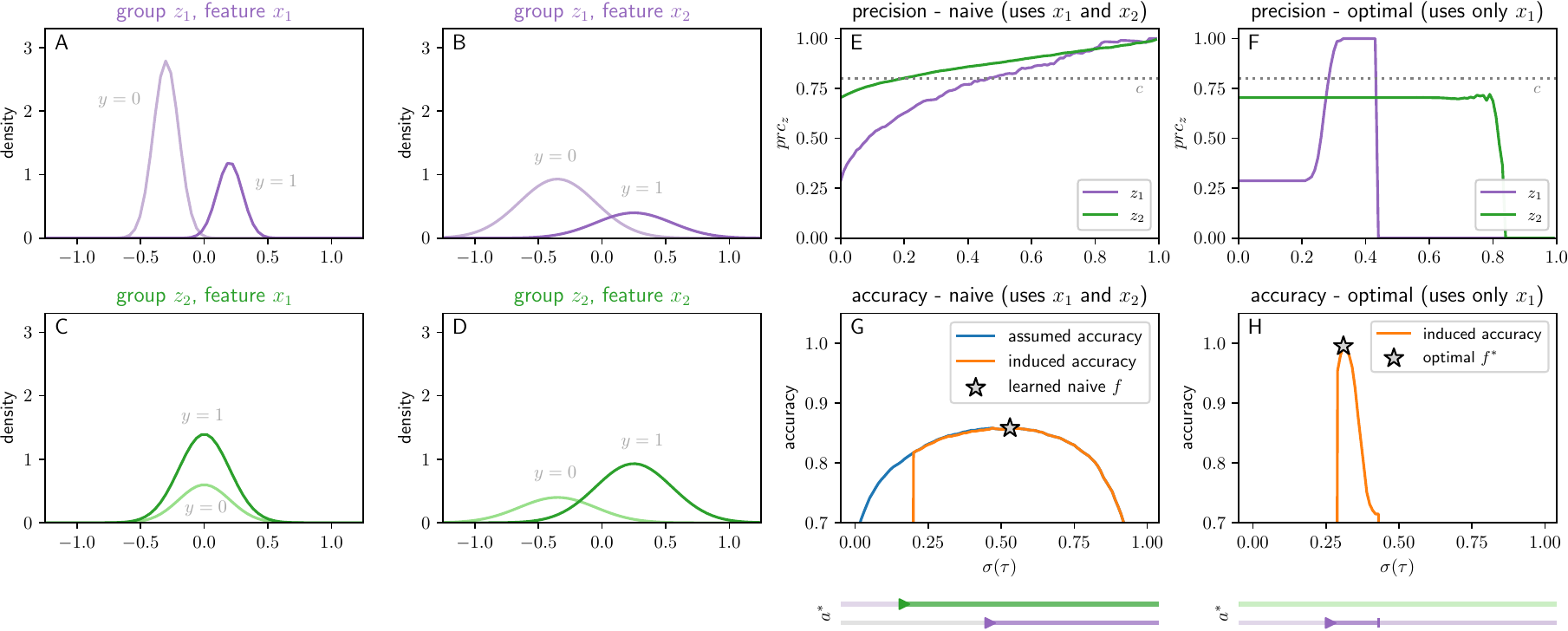}
\caption{
\textbf{Semi-strategic suboptimallity due to application order constraints.}
This example includes two groups $\grp_1,\grp_2$ and two features $\xbar_1,\xbar_2$,
constructed so that a {\naive}ly-trained classifier $f$ would rely mostly on $\xbar_2$ \textbf{(B,D)},
whereas the optimal strategic classifier $f^*$ is to use only $\xbar_1$ \textbf{(A,C)}.
This is because a varying the threshold on $f$ induces an ordering
$\grp_2 \preceq_\score \grp_1$ \textbf{(E)}, enabling a maximal accuracy of $\approx0.86$ \textbf{(G)}.
In contrast, $f^*$ induces $a^*_1=1$ but $a^*_2=0$ \textbf{(F)},
which is impossible under $\preceq_\score$, and enables an accuracy of $\approx0.99$.
This is since only $\xbar_1$ is used and only for $\grp_1$ \textbf{(A)},
while avoiding its ineffectiveness for $\grp_2$ \textbf{(C)},
and not needing to further rely on the less informative $\xbar_2$ \textbf{(B)}.
For further details see the text in Appendix~\ref{appx:synth_semi-strat}.
\squeeze
}
\label{fig:synth_ordering}
\end{figure}

\subsection{Semi-strategic learning -- ordering constraints} \label{appx:synth_semi-strat}
The goal of this experiment is to demonstrate the limitations of semi-strategic learning,
namely first training a score function $\score$ using standard methods (i.e., {\naive}ly), and then strategically tuning the threshold $\thresh$.
In particular, we show how the fact that semi-strategic learning induces an ordering $\preceq_\score$ on applications, derived from the learned $\score$,
prevents this approach from obtaining the optimal classifier, whose applications do not comply with $\preceq_\score$.

Again denoting by $\xbar = x \setminus \grp$ the non-group features,
here we use $\xbar \in \R^2$.
We define $K=2$ groups, where each group $\grp$ is associated with a data distribution $p(\xbar,y|\grp)=(\xbar_1\mid y,z)p(\xbar_2\mid y,z)p(y\mid \grp)$.
Note this means that $\xbar_1\perp \xbar_2$.
We set:
\begin{align*}
& \xbar_1 \mid y=0, \grp_1 \sim \mathcal{N}(-0.3,\,0.1), \quad    
\xbar_1 \mid y=1, \grp_1 \sim \mathcal{N}(0.2,\,0.1),  \\
& \xbar_2 \mid y=0, \grp_1 \sim \mathcal{N}(-0.35,\,0.3), \quad    
\xbar_2 \mid y=1, \grp_1 \sim \mathcal{N}(0.25,\,0.3),    \\
& \xbar_1 \mid y=0, \grp_2 \sim \mathcal{N}(0,\,0.2), \quad    
\xbar_1 \mid y=1, \grp_2 \sim \mathcal{N}(0,\,0.2),  \\
& \xbar_2 \mid y=0, \grp_2 \sim \mathcal{N}(-0.35,\,0.3), \quad    
\xbar_2 \mid y=1, \grp_2 \sim \mathcal{N}(0.25,\,0.3), \\
& \base_1 = 0.3, \quad \base_2 = 0.7
\end{align*}
Fig.~\ref{fig:synth_ordering} (A-D) visualize these distributions.
We also set $p(\grp=\grp_1)=0.1, p(\grp=\grp_2)=0.9$, and sampled 10,000 examples.

The idea underlying this construction is the following:
In terms of features, $x_2$ is generally informative of $y$, albeit noisy
(see how the class-conditional distributions overlap in Fig.~\ref{fig:synth_ordering} (B) and (D)).
This holds for both $\grp_1$and $\grp_2$.
In contrast, $x_2$ is highly informative of $y$ but \emph{only for $\grp_1$} (see (A),
whereas for $\grp_2$ it is completely uninformative (note how the distributions in (C) fully overlap).
Because a \naive\ learner is unaware of application, it learns a score function that relies primarily on the informative feature---namely $x_2$.
This, coupled with the higher base rate of $\grp_2$, results in $\grp_2$ applying before $\grp_1$ once the threshold is increased, i.e., $\grp_2 \preceq_\score \grp_1$ (see (E)).
As a result, the only feasible application assignments are:
\begin{equation} \label{eq:synth_ordering_apps}
a^*_1=0, a^*_2=0, \qquad\quad
a^*_1=0, a^*_2=1, \qquad\quad
a^*_1=1, a^*_2=1
\end{equation}
Under the optimal semi-strategic model (trained with logistic regression, threshold then optimized on the induced distribution), applications turned out to be $a^*_1=1, a^*_2=1$ (see (G)).
However, \emph{neither} of the applications assignments in Eq.~\eqref{eq:synth_ordering_apps} are optimal.
This is because the optimal solution is to completely discard $x_2$,
and use only $x_1$, but make sure that only $\grp_1$ applies (see (F)).
By learning $f(\xbar_1)$ and varying $\thresh$,
we ensure that the only information used to predict $y$ derives from
$p(\xbar_1 \,|\, \grp_1)$ (see A),
and dodges the uninformative $p(\xbar_1 \,|\, \grp_2)$ (see C).
This classifier induces an application profile of
$a^*_1=1, a^*_2=0$ (impossible under $\preceq_\score$),
which in turns provides it with almost perfect induced accuracy (see (H)).

Note that the semi-strategic classifier was `correct' in its assumed accuracy (see G).%
\footnote{Technically this is since $a^*_1=1, a^*_2=1$, which aligns with the {\naive} assumption, but note that assumed accuracy almost precisely matches induced accuracy also for lower thresholds where $a^*_1=0, a^*_2=1$, and so this holds more generally.}
Hence, its failure does not derive from a disparity between train and test distribution, but rather, from its inability to anticipate and account for strategic self-selective behavior during training.

\subsection{Details for illustration of precision curves under semi-strategic learning (Fig.~\ref{fig:precision_curves+ordering})} \label{appx:exp_details_synth}
The following describes the experimental setup of the synthetic example illustrated in Sec.~\ref{sec:semi-strat}.
Denoting by $\xbar = x \setminus \grp$ the non-group features,
here we use $\xbar \in \R$.
In this experiment, we created data for $K=10$ groups.
Each group $\grp_i$ is associated with a data distribution
$p(\xbar, y \,|\, \grp_i)=p(\xbar \,|\,y,\grp_i) p(y\,|\,\grp_i)$, with Gaussian class-conditional distributions:
$\xbar \mid y=0 \sim \mathcal{N}(a_{i}^{-},\,{\sigma^{2}}_{i}^{-})$ and
$\xbar \mid y=1 \sim \mathcal{N}(a_{i}^{+},\,{\sigma^{2}}_{i}^{+})$, and a base rate $\base_i=p(y=1\,|\,\grp_i)$, with $p(y=0\,|\,\grp_i)= 1-\base_i$.
These per-group distributions were created by randomly sampling their parameters:
$a_{i}^{-}\sim\mathcal{U}(0,0.5),\,{\sigma^{2}}_{i}^{-}\sim\mathcal{U}(0.2,0.5),\,a_{i}^{+}\sim\mathcal{U}(0.5,1),\,{\sigma^{2}}_{i}^{+}\sim\mathcal{U}(0.2,0.5),\, \base_i\sim\mathcal{U}(0,1)$.
We sampled a total of 10,000 examples, with 1,000 examples per group, and `{\naive}ly' trained a linear model using logistic regression on the full dataset including all groups.
Then we fixed the cost $c=0.8$, and varied the decision threshold $\thresh$ between $[0,1]$ (applied to the probabilistic scores of the learned classifier),
measuring $\prc_{i}$ and $a^{*}_{i}$ of the groups for each value of $\thresh$.


\section{Experimental details} \label{appx:exp_details}

\subsection{Data and preprocessing}
\subsubsection{Adult}

\paragraph{Data description.}
This dataset contains features based on census data from the 1994 census database that describe demographic and financial data. There are 13 features,  6 of which are categorical and the others numerical. 
The binary label is whether a person's income exceeds \$50k.
The dataset includes a total of 48,842 entries,
76\% of which are labeled as negative.
The data is publicly available at \url{https://archive.ics.uci.edu/dataset/2/adult}.

\paragraph{Preprocessing and features.}
We used the following numerical features:
\feature{age}, \feature{final\_weight}, \feature{education\_num}, \feature{capitol\_gain}, \feature{capitol\_loss}, \feature{hours\_per\_week}.
All such features were normalized to be in $[0,1]$.
We used the following categorical features:
\feature{work\_class}, \feature{marital\_status}, \feature{relationship}, \feature{race}, \feature{sex} and \feature{occupation}.
All of these were transformed into one-hot binary features representations.
The \feature{occupation} feature was reduced from 15 values to 5 based on similarity.  
We chose \feature{race} as the group variable $\grp$.
The dataset includes 5 race categories, but since
Amer-Indian-Eskimo category has very few entries, it was combined with the "Other" category.
We did not use two features:
\feature{education}, since it correlates perfectly to education\_num (which we do use);
and \feature{native\_country}, since it takes on many possible values and is uninformative of the label.
The final number of features used for this setting is $d=\red{???}$.

\paragraph{Sub-sampling.}
In the original data, the category "white" (coded as $\grp_3$) consists of 86\% of all examples, which is highly imbalanced.
To create a more balanced experimental setting,
we removed at random 75\% of the examples in that group.
After this step, the number of examples per groups were:
$\grp_0: 1303; \,\, \grp_1: 4228; \,\, \grp_2: 788; \,\,$ and $\grp_3: 9726$.

\subsubsection{Bank}
\paragraph{Data description.}
This dataset describes users and results of direct marketing campaigns of a Portuguese banking institution. There are 16 features, 6 of which are categorical and the others numerical. The binary label is whether a person subscribes to a proposed term deposit.
The dataset includes a total of 45,211 entries, 88\% of which are labeled as negative.
The data is publicly available at \url{https://www.kaggle.com/datasets/prakharrathi25/banking-dataset-marketing-targets}.

\paragraph{Preprocessing and features.}
We used the following numerical features:
\feature{default}, \feature{balance}, \feature{housing}, \feature{loan}, \feature{contact}, \feature{day}, \feature{month}, \feature{duration}, \feature{campaign}, \feature{pdays} and \feature{previous}.
All such features were normalized to be in $[0,1]$.
We used the following categorical features:
\feature{job}, \feature{marital}\_status, \feature{education}, \feature{contact}, and \feature{poutcome},
and transformed them into one-hot binary features representations. 
Features \feature{day} and \feature{month} where not used given that they have no meaningful relation to the label.
We chose \feature{job} to determine groups features $\grp$,
and removed the categories \feature{unknown} and \feature{housemaid} (which had few entries) to remain with 10 groups.
The final number of features used for this setting is $d=\red{???}$.

\paragraph{Sub-sampling.}
The original data is reasonably balanced across groups,
where the number of examples per group are:
$\grp_0: 5,175; \,\,
\grp_1: 9,732; \,\,
\grp_2: 1,487; \,\,
\grp_3: 9,458; \,\,
\grp_4: 2,264; \,\,
\grp_5: 1,579; \,\,
\grp_6: 4,154; \,\,
\grp_7: 9,38; \,\,
\grp_8: 7,597; \,\,
\grp_9: 1,303$.
However, labels are highly-imbalanced, with only 11.78\% positive.
As a result, we observed that learning in general sometimes converges to a solution that predicts $\yhat=0$ always.
To circumvent this and make for a more interesting experimental settings,
we globally down-sampled 30\% of all negative examples.
This increased the base rate to 16.02\%---which is still relatively low, but enabled more meaningful learning solutions.
Results in the main body are for this setting, but for completeness, we include full results for both the original and down-sampled datasets in Appendix~\ref{appx:additional_exps}.

\subsection{Splits and repetitions}
All experiments use a 70-30 split for partitioning the data into train and test sets. 
For \adult, this amounts to $\sim$11,000 train examples and 
$\sim$4,800 test examples. 
For \bank, this amounts to $\sim$30,000 train examples and 
$\sim$13,000 test examples for the original dataset, and
$\sim$22,000 train examples and $\sim$9,600 test examples for the down-sampled variant.
Overall we did not see evidence of overfitting, and hence had no need for a validation set.
We experimented with 10 random splits and report results averaged over these splits, including standard errors.

\subsection{Methods}
Our experiments compares three main learning approaches, corresponding to those discussed in Sec.~\ref{sec:analysis}.
In our experiments we use linear classifiers as the underlying hypothesis class, although generally any differentiable class could be used.
\begin{itemize}[leftmargin=1em,topsep=0em,itemsep=0.1em]
\item 
\naivemthd: A naive learner that does not account for strategic behaviour and simply optimizes for accuracy using a conventional learning approach
(see Sec.~\ref{sec:naive}).
In particular, we optimize the log-loss using gradient descent.

\item
\semi: A semi-strategic approach in which we first train a classifier using the \naivemthd\ method, and then strategically choose the threshold that maximizes induced accuracy on the training set (see Sec.~\ref{sec:semi-strat}).
Because the strategic aspect reduces to solving a uni-dimensional problem,
this is implemented by line search over the feasible range of thresholds.
Because this does not require taking gradient steps, we use the `hard' induced accuracy metric (i.e., 0-1 accuracy on hard predictions $\yhat$) as the tresholding criterion. In this sense, \semi\ has an advantage over \strat.

\item
\strat: Optimizes our proposed strategically-aware learning objective in (Eq.~\eqref{eq:objective_induced_smoothed}). 
The objective is designed to be a differential proxy for induced accuracy,
and is optimized using gradient descent (for details see Sec.~\ref{sec:method}).
We consider three variants of this approach, as discussed in Sec.~\ref{sec:analysis_strat}:
\begin{itemize}[leftmargin=1em,topsep=0em,itemsep=0.1em]
\item
\stratx:  Makes use of all available features in $x$, i.e., has the form $f(x)=\inner{w,x}+b$ where $w \in \R^d$ and $b \in \R$.
Note that in particular, since $\grp \subset x$, and since we represent $\grp$ as a one-hot vector of size $K$,
this means that the model has individual per-group offset terms $w_\grp$ (note this makes $b$ redundant). This allows it to effectively set group-specific thresholds---from which it derives much of its power to influence applications.

\item
\stratnoz: Uses only non-group features, i.e., $\xbar = x \setminus \grp$.
This serves as a simple heuristic for approximating independence $\yhat \perp \grp$ (see Eq.~\eqref{eq:fairness_constraint})---although of course this provides no guarantees, since the remaining features $\xbar$ can still be informative of $\grp$.
Nonetheless, this approach is still much less expressive:
this is since the model is now $f(x)=\inner{w,\xbar}+b$, where $w \in R^{d-K}$.
In particular, this means that $f$ does \emph{not} have group-specific offsets $w_\grp$, and can influence applications only globally by varying the global offset $b$.

\item
\stratindep: Uses only non-group features, but additionally penalizes for violation of the statistical parity independence constraints in Eq.~\eqref{eq:fairness_constraint}. Technically, this is achieved by adding to the objective the regularization term $R_\perp$ from Eq.~\eqref{eq:regularizer_penalty} as a `soft' constraint which encourages $\yhat \perp \grp$.

\end{itemize}
\end{itemize}




\subsection{Training, tuning, and optimization}

\paragraph{Implementation.}
All code was implemented in python,
and the learning framework was implemented using Pytorch.
For proper comparison, all methods were optimized using the same underlying implementation framework, with \semi\ and \strat\ implemented as subclasses of \naivemthd\ and using the same code base.
To ensure validity we also made sure that the performance of our implementation of \naivemthd\ matches that of a standard sklearn implementation on a subset of the experimental settings.

\paragraph{Optimization.}
Training for all methods and settings was done using vanilla gradient descent (full batch).
We used a learning rate of 0.1, chosen manually to provide fast convergence while maintaining a smooth learning curve (we observes that values $>0.5$ result in sporadic instability).
Lower values resulted in similar results, but converged slower.
We ran for a fixed number of 30,000 epochs since in most instances this was sufficient for the objective to converge sufficiently and for other metrics (train accuracy and precision) to be relatively stable.
We also experimented with shorter and longer runs, but this did not seem to effect results generally.

\paragraph{Initialization.}
All models were initialized with Gaussian noise.
For \strat, which is non-convex, we observed that some initializations converged to highly suboptimal solutions---this occurred when learning `committed' to advancing the precision of a low-quality group at the onset, and was unable to correct for this. To compensate for this, we ran with different initialization (5 for \adult, 10 for \bank) and chose model obtaining the highest induced accuracy on the training set.
We validated that \naivemthd\ (and therefore \semi), which is convex, did not benefit from multiple initializations, and therefore used a single initialization.

\paragraph{Hyperparameters.}
We used the following hyperparameters:
\begin{itemize}[leftmargin=1em,topsep=0em,itemsep=0.1em]
\item 
Temperature $\tau_\mathrm{app}$ for the application sigmoid $\sig$ in Eq.~\eqref{eq:sigmoid}: 5

\item 
Temperature $\tau_\mathrm{prec}$ for the precision proxy $\softprc$ in Eq.~\eqref{eq:precision_proxy}: 5

\item 
Temperature $\tau$ for standard sigmoid $\sigalt$: 2

\item
Cost tolerance $\varepsilon$: 0.02 for \adult, 0.05 for \bank.

\item Regularization coefficient $\lambda_{\mathrm{app}}$ for $R_{\mathrm{app}}$
in Eq.~\eqref{eq:app&prec_penalty}: 1/6 for \strat\ and \stratindep\ for \adult, 1/6 for \strat\ and 1/64 for \stratindep\ for \bank.

\item Regularization coefficient $\lambda_{\perp}$ for $R_{\perp}$
in Eq.~\eqref{eq:regularizer_penalty} (used only for \stratindep):
For \adult, we set 8 for the lowest cost $\costu=0.65$, and increasing linearly up to 16 up to $\costu=0.85$.
For \bank\ we used 100.%
\footnote{In our implementation, to avoid numerical instability, we multiplied the expectation terms inside the mean squares operator by 10, and used a coefficient of 10, which is equivalent to a coefficient of 100 without scaling.}

\end{itemize}

Temperature parameters where chosen mostly to provide fast convergence while ensuring that gradients do not explode.
This choice is not overly sensitive, although we did observe that excessively low values resulted in premature convergence.
Cost tolerance was chosen to be slightly higher for \bank\ since here we did observe mild overfitting in application outcomes---which is precisely the reason for our use of a tolerance term.
Regularization for applications was chosen to be as small as possible yet ensure the feasibility of applications and precisions.
Note that \stratindep\ requires a different cofficients since it must balance application reagularization with the independence penalty term.
The latter was chosen to ensure that the mean squared distance is sufficiently low so that independence is reasonably-well approximated.


\paragraph{Compute and runtime.}
The main experiment was run on a CPU cluster of
AMD EPYC 7713 machines (1.6 Ghz, 256M, 128 cores).
A typical epoch was timed at roughly 0.05 seconds per epoch for \adult, and 0.07 for \bank.
Thus, a single experimental instance (i.e., for a single method, cost, split, and initialization) completes in approximately 20-30 minutes for \adult\ and 30-45 minutes for \bank.


\section{Additional experimental results} \label{appx:additional_exps}

\begin{table}[h!]
\centering
\caption{\textbf{Extended experimental results.}
Results show:
induced accuracy ($\pm$stderr),
number of applying groups,
and the $r^2$ between the ideal $\preceq_\base$ and the actual ranking based on $\prc_\grp$.
Parentheses/dashes mark settings in which there were no applications in some/all splits (out of 10).
\squeeze
}
\begin{adjustbox}{width=0.75\textwidth}
\begin{tabular}{clrccccccccccc}
  &   &   & \multicolumn{3}{c}{\textbf{\adult}} &   & \multicolumn{3}{c}{\textbf{\bank} (30\% negs.)} &   & \multicolumn{3}{c}{\textbf{\bank} (original)} \\
\cmidrule{4-6}\cmidrule{8-10}\cmidrule{12-14}  &   &   & ind. acc. & apply & rank $r^2$ &   & ind. acc. & apply & rank $r^2$ &   & ind. acc. & apply & rank $r^2$ \\
\cmidrule{4-6}\cmidrule{8-10}\cmidrule{12-14}\multirow{5}[2]{*}{$c=0.65$} & \naivemthd &   & 85.5\tiny{\,$\pm0.1$} & 4.0/4 & 0.219 &   & 87.4\tiny{\,$\pm0.1$} & 9.1/10 & 0.066 &   & 89.4\tiny{\,$\pm0.2$} & 5.4/10 & 0.066 \\
  & \semi &   & 86.3\tiny{\,$\pm0.5$} & 2.8/4 & 0.168 &   & 89.4\tiny{\,$\pm0.4$} & 2.1/10 & 0.054 &   & 92.7\tiny{\,$\pm0.1$} & 1.0/10 & 0.078 \\
  & \stratx &   & 90.8\tiny{\,$\pm0.3$} & 1.6/4 & 0.388 &   & 90.4\tiny{\,$\pm0.5$} & 1.6/10 & 0.183 &   & 92.2\tiny{\,$\pm0.2$} & 1.3/10 & 0.216 \\
  & \stratnoz &   & 85.4\tiny{\,$\pm0.2$} & 3.1/4 & 0.186 &   & 87.5\tiny{\,$\pm0.3$} & 7.8/10 & 0.243 &   & 90.6\tiny{\,$\pm0.4$} & 3.6/10 & 0.089 \\
  & \stratindep &   & 81.7\tiny{\,$\pm0.5$} & 0.8/4 & 0.708 &   & 88.1\tiny{\,$\pm0.4$} & 1.5/10 & 0.442 &   & 91.7\tiny{\,$\pm0.4$} & 1.4/10 & 0.373 \\
\cmidrule{4-6}\cmidrule{8-10}\cmidrule{12-14}  &   &   &   &   &   &   &   &   &   &   &   &   &  \\
\cmidrule{4-6}\cmidrule{8-10}\cmidrule{12-14}\multirow{5}[2]{*}{$c=0.675$} & \naivemthd &   & 85.5\tiny{\,$\pm0.1$} & 3.9/4 & 0.219 &   & 83.5\tiny{\,$\pm0.4$} & 7.8/10 & 0.066 &   & 84.8\tiny{\,$\pm1$} & 3.3/10 & 0.066 \\
  & \semi &   & 86.7\tiny{\,$\pm0.6$} & 2.5/4 & 0.154 &   & 89.5\tiny{\,$\pm0.3$} & 1.4/10 & 0.125 &   & 92.7\tiny{\,$\pm0.1$} & 1.1/10 & 0.067 \\
  & \stratx &   & 90.8\tiny{\,$\pm0.5$} & 1.3/4 & 0.149 &   & 90.4\tiny{\,$\pm0.5$} & 1.7/10 & 0.167 &   & 92.2\tiny{\,$\pm0.2$} & 1.1/10 & 0.304 \\
  & \stratnoz &   & 85.4\tiny{\,$\pm0.3$} & 3.0/4 & 0.262 &   & 87.5\tiny{\,$\pm0.3$} & 7.7/10 & 0.172 &   & 90.3\tiny{\,$\pm0.6$} & 2.9/10 & 0.098 \\
  & \stratindep &   & 86.4\tiny{\,$\pm0.6$} & 0.2/4 & 0.851 &   & 87.5\tiny{\,$\pm0.5$} & 1.3/10 & 0.429 &   & 92.2\tiny{\,$\pm0.2$} & 1.0/10 & 0.330 \\
\cmidrule{4-6}\cmidrule{8-10}\cmidrule{12-14}  &   &   &   &   &   &   &   &   &   &   &   &   &  \\
\cmidrule{4-6}\cmidrule{8-10}\cmidrule{12-14}\multirow{5}[2]{*}{$c=0.7$} & \naivemthd &   & 85.2\tiny{\,$\pm0.3$} & 3.0/4 & 0.219 &   & 80.9\tiny{\,$\pm0.8$} & 5.4/10 & 0.066 &   & 82.7\tiny{\,$\pm1.4$} & 1.7/10& 0.066 \\
  & \semi &   & 87.4\tiny{\,$\pm0.6$} & 2.1/4 & 0.135 &   & 90.0\tiny{\,$\pm0.4$} & 1.5/10 & 0.068 &   & 92.3\tiny{\,$\pm0.2$} & 1.2/10 & 0.045 \\
  & \stratx &   & 91.1\tiny{\,$\pm0.5$} & 1.1/4 & 0.076 &   & 90.1\tiny{\,$\pm0.4$} & 1.8/10 & 0.177 &   & 92.1\tiny{\,$\pm0.2$} & 1.1/10 & 0.230 \\
  & \stratnoz &   & 86.0\tiny{\,$\pm0.4$} & 2.5/4 & 0.244 &   & 87.9\tiny{\,$\pm0.6$} & 6.7/10 & 0.170 &   & 91.1\tiny{\,$\pm0.5$} & 2.6/10 & 0.057 \\
  & \stratindep &   & 86.5\tiny{\,$\pm0.1$} & 0.6/4 & 0.85 &   & 87.3\tiny{\,$\pm0.4$} & 0.9/10 & 0.343 &   & 92.0\tiny{\,$\pm0.3$} & 1.1/10 & 0.267 \\
\cmidrule{4-6}\cmidrule{8-10}\cmidrule{12-14}  &   &   &   &   &   &   &   &   &   &   &   &   &  \\
\cmidrule{4-6}\cmidrule{8-10}\cmidrule{12-14}\multirow{5}[2]{*}{$c=0.725$} & \naivemthd &   & 84.5\tiny{\,$\pm0.7$} & 1.8/4 & 0.219 &   & 76.7\tiny{\,$\pm1.6$} & 2.3/10 & 0.066 &   & (77.9) & 0.5/10 & (0.066) \\
  & \semi &   & 87.9\tiny{\,$\pm0.9$} & 1.7/4 & 0.065 &   & 90.5\tiny{\,$\pm0.3$} & 1.1/10 & 0.107 &   & 91.4\tiny{\,$\pm0.8$} & 1.1/10 & 0.048 \\
  & \stratx &   & 90.4\tiny{\,$\pm0.4$} & 1.0/4 & 0.003 &   & 90.1\tiny{\,$\pm0.4$} & 1.8/10 & 0.193 &   & 92.1\tiny{\,$\pm0.2$} & 1.2/10 & 0.296 \\
  & \stratnoz &   & 87.0\tiny{\,$\pm0.8$} & 1.9/4 & 0.389 &   & 87.4\tiny{\,$\pm0.7$} & 5.9/10 & 0.160 &   & 91.4\tiny{\,$\pm0.4$} & 2.0/10 & 0.086 \\
  & \stratindep &   & 86.7\tiny{\,$\pm0.7$} & 0.4/4 & 0.659 &   & 87.4\tiny{\,$\pm0.3$} & 1.0/10 & 0.379 &   & 92.4\tiny{\,$\pm0.1$} & 1.2/10 & 0.311 \\
\cmidrule{4-6}\cmidrule{8-10}\cmidrule{12-14}  &   &   &   &   &   &   &   &   &   &   &   &   &  \\
\cmidrule{4-6}\cmidrule{8-10}\cmidrule{12-14}\multirow{5}[2]{*}{$c=0.75$} & \naivemthd &   & (90.0) & 0.4/4 & (0.219) &   & 78.7\tiny{\,$\pm2.8$} & 1.0/10 & 0.066 &   & (77.0) & 0.1/10 & (0.066) \\
  & \semi &   & 88.0\tiny{\,$\pm1$} & 1.7/4 & 0.125 &   & 90.0\tiny{\,$\pm0.3$} & 1.5/10 & 0.080 &   & 90.4\tiny{\,$\pm0.6$} & 1.7/10 & 0.062 \\
  & \stratx &   & 91.0\tiny{\,$\pm0.6$} & 1.0/4 & 0.005 &   & 89.8\tiny{\,$\pm0.3$} & 1.3/10 & 0.275 &   & 91.8\tiny{\,$\pm0$} & 1.0/10 & 0.327 \\
  & \stratnoz &   & 88.9\tiny{\,$\pm0.6$} & 1.2/4 & 0.094 &   & 87.2\tiny{\,$\pm1.4$} & 4.5/10 & 0.091 &   & 91.6\tiny{\,$\pm0.3$} & 1.7/10 & 0.056 \\
  & \stratindep &   & 87.4\tiny{\,$\pm0.6$} & 0.6/4 & 0.675 &   & 88.2\tiny{\,$\pm0.4$} & 1.4/10 & 0.424 &   & 92.4\tiny{\,$\pm0$} & 1.1/10 & 0.244 \\
\cmidrule{4-6}\cmidrule{8-10}\cmidrule{12-14}  &   &   &   &   &   &   &   &   &   &   &   &   &  \\
\cmidrule{4-6}\cmidrule{8-10}\cmidrule{12-14}\multirow{5}[2]{*}{$c=0.775$} & \naivemthd &   & (87.8) & 0.1/4 & (0.219) &   & (86.7) & 0.1/10 & (0.066) &   & (77.0) & 0.1/10 & (0.066) \\
  & \semi &   & 88.7\tiny{\,$\pm0.8$} & 1.5/4 & 0.327 &   & 88.6\tiny{\,$\pm0.7$} & 1.7/10 & 0.111 &   & 90.7\tiny{\,$\pm0.8$} & 1.3/10 & 0.096 \\
  & \stratx &   & 91.1\tiny{\,$\pm0.6$} & 1.0/4 & 0.005 &   & 89.7\tiny{\,$\pm0.3$} & 1.3/10 & 0.191 &   & 91.8\tiny{\,$\pm0$} & 1.0/10 & 0.352 \\
  & \stratnoz &   & 89.0\tiny{\,$\pm0.6$} & 1.0/4 & 0.030 &   & 87.5\tiny{\,$\pm1.5$} & 2.3/10 & 0.095 &   & 91.3\tiny{\,$\pm0.8$} & 1.2/10 & 0.094 \\
  & \stratindep &   & 87.4\tiny{\,$\pm0.8$} & 0.4/4 & 0.505 &   & 88.1\tiny{\,$\pm0.4$} & 1.2/10 & 0.462 &   & 92.2\tiny{\,$\pm0$} & 1.1/10 & 0.200 \\
\cmidrule{4-6}\cmidrule{8-10}\cmidrule{12-14}  &   &   &   &   &   &   &   &   &   &   &   &   &  \\
\cmidrule{4-6}\cmidrule{8-10}\cmidrule{12-14}\multirow{5}[2]{*}{$c=0.8$} & \naivemthd &   & (87.8) & 0.1/4 & (0.219) &   & - & 0/10 & - &   & - & 0/10 & - \\
  & \semi &   & 90.1\tiny{\,$\pm0.5$} & 1.2/4 & 0.300 &   & 86.4\tiny{\,$\pm0.6$} & 2.4/10 & 0.100 &   & 90.4\tiny{\,$\pm0.8$} & 1.5/10 & 0.114 \\
  & \stratx &   & 90.5\tiny{\,$\pm0.5$} & 1.0/4 & 0.003 &   & 88.7\tiny{\,$\pm0.4$} & 1.1/10 & 0.238 &   & 92.1\tiny{\,$\pm0.2$} & 1.0/10 & 0.321 \\
  & \stratnoz &   & 89.0\tiny{\,$\pm0.5$} & 1.0/4 & 0.029 &   & 87.0\tiny{\,$\pm1.1$} & 1.5/10 & 0.121 &   & 91.2\tiny{\,$\pm0.8$} & 1.3/10 & 0.061 \\
  & \stratindep &   & 88.4\tiny{\,$\pm1.2$} & 0.4/4 & 0.537 &   & 88.2\tiny{\,$\pm0.5$} & 1.3/10 & 0.361 &   & 91.9\tiny{\,$\pm0.3$} & 1.2/10 & 0.150 \\
\cmidrule{4-6}\cmidrule{8-10}\cmidrule{12-14}  &   &   &   &   &   &   &   &   &   &   &   &   &  \\
\cmidrule{4-6}\cmidrule{8-10}\cmidrule{12-14}\multirow{5}[2]{*}{$c=0.825$} & \naivemthd &   & - & 0/10 & - &   & - & 0/10 & - &   & - & 0/10 & - \\
  & \semi &   & 90.3\tiny{\,$\pm0.3$} & 1.1/4 & 0.411 &   & 85.6\tiny{\,$\pm0.4$} & 2.3/10 & 0.107 &   & 90.2\tiny{\,$\pm0.8$} & 1.2/10 & 0.081 \\
  & \stratx &   & 90.7\tiny{\,$\pm0.5$} & 1.0/4 & 0.003 &   & 89.0\tiny{\,$\pm0.3$} & 1.1/10 & 0.292 &   & 92.2\tiny{\,$\pm0.2$} & 1.0/10 & 0.285 \\
  & \stratnoz &   & 88.9\tiny{\,$\pm0.5$} & 1.0/4 & 0.180 &   & 87.0\tiny{\,$\pm0.7$} & 1.6/10 & 0.150 &   & 90.8\tiny{\,$\pm0.8$} & 1.3/10 & 0.075 \\
  & \stratindep &   & 88.7\tiny{\,$\pm1.1$} & 0.6/4 & 0.354 &   & 88.3\tiny{\,$\pm0.4$} & 1.1/10 & 0.336 &   & 92.0\tiny{\,$\pm0.3$} & 1.1/10 & 0.281 \\
\cmidrule{4-6}\cmidrule{8-10}\cmidrule{12-14}  &   &   &   &   &   &   &   &   &   &   &   &   &  \\
\cmidrule{4-6}\cmidrule{8-10}\cmidrule{12-14}\multirow{5}[2]{*}{$c=0.85$} & \naivemthd &   & - & 0/10 & - &   & - & 0/10 & - &   & - & 0/10 & - \\
  & \semi &   & 90.1\tiny{\,$\pm0.2$} & 1.2/4 & 0.560 &   & 86.9\tiny{\,$\pm0.6$} & 2.1/10 & 0.067 &   & 90.2\tiny{\,$\pm0.5$} & 1.0/10 & 0.121 \\
  & \stratx &   & 91.0\tiny{\,$\pm0.6$} & 1.0/4 & 0.005 &   & 89.0\tiny{\,$\pm0.2$} & 1.0/10 & 0.344 &   & 91.8\tiny{\,$\pm0.3$} & 1.0/10 & 0.314 \\
  & \stratnoz &   & 89.2\tiny{\,$\pm0.6$} & 1.0/4 & 0.219 &   & 86.4\tiny{\,$\pm0.8$} & 1.0/10 & 0.102 &   & 91.1\tiny{\,$\pm0.8$} & 1.2/10 & 0.066 \\
  & \stratindep &   & 89.1\tiny{\,$\pm0.8$} & 1/4 & 0.198 &   & 87.6\tiny{\,$\pm0.7$} & 1.3/10 & 0.341 &   & 92.1\tiny{\,$\pm0.2$} & 1.0/10 & 0.183 \\
\cmidrule{4-6}\cmidrule{8-10}\cmidrule{12-14}\end{tabular}%
\end{adjustbox}
\label{tbl:results}%
\end{table}%


\end{document}